\documentclass{amsart}  
\usepackage{mathabx,amsmath,amssymb,latexsym}
\usepackage{comment,nicefrac}
\usepackage{tikz,pgfplots}
\usepackage{hyperref}
\usetikzlibrary{arrows,snakes,backgrounds}
\usepackage[export]{adjustbox}
\theoremstyle{plain}             % oder andere
\newtheorem{theorem}{Theorem}[section]
\newtheorem{lemma}[theorem]{Lemma}
\newtheorem{corollary}[theorem]{Corollary}

\newtheorem{definition}[theorem]{Definition}
\newtheorem{remark}[theorem]{Remark}
\newtheorem{assumption}[theorem]{Assumption}
\usepackage[normalem]{ulem}
\usepackage{dsfont}
\newcommand{\be}{\begin{equation}}
\newcommand{\ee}{\end{equation}}
\newcommand{\Diff}{{\rm Diff}}
\newcommand{\KL}{{Karhunen-Lo\`{e}ve}}
\newcommand{\calD}{{\mathcal{D}}}
\newcommand{\calE}{{\mathcal{E}}}
\newcommand{\calF}{{\mathcal{F}}}
\newcommand{\calL}{{\mathcal{L}}}
\newcommand{\calR}{{\mathcal{R}}}
\newcommand{\calS}{{\mathcal{S}}}
\newcommand{\bcalS}{{\boldsymbol{\mathcal{S}}}}
\newcommand{\calX}{{\mathcal{X}}}
\newcommand{\bcalX}{{\boldsymbol{\mathcal{X}}}}
\newcommand{\calY}{{\mathcal{Y}}}
\newcommand{\bcalY}{{\boldsymbol{\mathcal{Y}}}}
\newcommand{\calZ}{{\mathcal{Z}}}

\newcommand{\calG}{{\mathcal{G}}}
\newcommand{\bcalG}{{\boldsymbol{\mathcal{G}}}}
\newcommand{\bG}{{\boldsymbol{G}}}

\newcommand{\bsPsi}{{\boldsymbol{\Psi}}}

\newcommand{\domain}{{D_\refd}}
\newcommand{\N}{\mathbb N}
\newcommand{\R}{\mathbb R}
\newcommand{\C}{\mathbb C}
\newcommand{\rF}{{\mathrm F}}
\newcommand{\rS}{{\mathrm S}}
\newcommand{\norm}[2][]{\|#2\|_{#1}}
\newcommand{\normlr}[2][]{\left\|#2\right\|_{#1}}
\newcommand{\set}[2]{\{#1 : #2\}}
\newcommand{\setlr}[2]{\left\{#1 : #2\right\}}
\newcommand{\dup}[2]{\langle #1,#2 \rangle}

\usepackage{enumitem}

\newcommand{\eps}{\varepsilon}
\newcommand{\isdef}{\mathrel{\mathrel{\mathop:}=}}

\newcommand{\trace}{\operatorname{tr}}
\renewcommand{\div}{\operatorname{div}}

\newcommand{\balpha}{{\boldsymbol{\alpha}}}

\newcommand{\bgamma}{{\boldsymbol{\gamma}}}

\newcommand{\bmc}{{\boldsymbol{c}}}

\newcommand{\bsnul}{{\boldsymbol{0}}}
\newcommand{\bsv} {{\boldsymbol{v}}}
\newcommand{\bsw} {{\boldsymbol{w}}}
\newcommand{\bsy} {{\boldsymbol{y}}}
\newcommand{\bsQ} {{\boldsymbol{Q}}}
\newcommand{\sym}{{\operatorname*{sym}}}
\newcommand{\refd}{{\operatorname*{ref}}}
\newcommand{\iso}{{\operatorname*{iso}}}

\renewcommand{\d}{\operatorname{d}\!}
\begin{document}
\title[Shape Operator Surrogates]{
Neural Shape Operator Surrogates -- 
\\ Expression Rate Bounds}
\author{Helmut Harbrecht}
\address{Helmut Harbrecht,
Departement Mathematik und Informatik,
Universit\"at Basel,
Spiegelgasse 1, 4051 Basel, Switzerland}
\email{helmut.harbrecht@unibas.ch}
\author{Christoph Schwab}
\address{Christoph Schwab,
Seminar for Applied Mathematics, 
ETH Z\"urich, 
R\"amistrasse 101, 8092 Z\"urich, Switzerland}
\email{schwab@math.ethz.ch}
\date{\today}
%==========================================================
\maketitle
%==========================================================
\begin{abstract}
We prove error bounds for operator surrogates 
of solution operators for partial differential and boundary integral equations
on families of domains
which are diffeomorphic 
to one common reference (or latent) 
domain $\domain$.
The pullback of the PDE to $\domain$ 
via affine-parametric shape encoding
produces a collection of 
holomorphic parametric PDEs on $\domain$.
Sufficient conditions for (uniformly with respect to the parameter)
well-posedness are given, 
implying existence, uniqueness and stability 
of parametric solution families on $\domain$.
We illustrate the abstract hypotheses by reviewing recent holomorphy results
for a suite of elliptic and parabolic PDEs.

Quantified parametric holomorphy 
implies existence of finite-parametric, discrete approximations of 
the parametric solution families with convergence rates
in terms of the number $N$ of parameters.
We obtain constructive proofs of existence
of Neural and Spectral Operator surrogates 
for the shape-to-solution maps 
with error bounds and convergence rate guarantees 
uniform on the collection of admissible shapes.
We admit principal-component shape encoders 
and frame decoders. 

Our results support in particular the (empirically reported) 
ability of neural operators to realize data-to-solution maps
for elliptic and parabolic PDEs and BIEs that generalize 
across parametric families of shapes.
\end{abstract}
%
%==========================================================
% \tableofcontents
%==========================================================
\section{Introduction}
%==========================================================
Numerous computational tasks in science and engineering 
require the repeated,
efficient numerical solutions of boundary value problems 
in physical domains of engineering interest. 
Frequently, a
particular operator equation (representing a mathematical model 
for the physical process under consideration) must be solved over 
\emph{collections of ``similar'' physical domains} and possibly also 
for large parametric sets of input data. 
Here, the notion ``similar'' could 
mean for example homeomorphic, or bi-Lipschitz equivalent.
Typical settings where this occurs are shape optimization, 
data-driven Bayesian inverse shape design, 
filtering and MCMC algorithms for shape identification.
Other applications include 
engineering design-, safety- and risk-analysis,
and uncertainty quantification.
Here, for a given physical model,
numerical solves of many instances of one operator equation 
-- usually in form of a partial differential equation (PDE) or 
a boundary integral equation (BIE) -- are required,
on collections of physical domains which are ``similar'' to each other.

A variety of techniques have been developed in recent years 
in order to re-use (parts of) earlier simulations of 
``neighboring'' models or problems.
We mention only operator preconditioning (\cite{HiptmOpPrec}),
multi-level sampling (\cite{GilesMLMCActa}),
techniques that can reduce total computing time for large batches 
of forward simulations significantly. 
See, e.g., \cite{li2023geometryinformedneuraloperatorlargescale}
for detailed numerical experiments from a problem for viscous, incompressible flow
in three space dimensions.

A novel paradigm which has emerged recently in Scientific Machine Learning
is the so-called \emph{operator learning}, 
see e.g.\ 
\cite{kovachki2024operatorlearningalgorithmsanalysis}
and the references there. 
Here, 
rather than re-running a numerical solver 
for each instance of the physical problem, 
one views the \emph{data-to-solution map}
(realized approximately, e.g.\ by a ``high-fidelity'' numerical solver) 
as a map between 
(infinite-dimensional) function spaces 
which should be ``emulated'' by a 
\emph{finite-parametric, numerical surrogate map}. 
In constructing such surrogate maps 
with certified performance on continua of inputs,
for reasons of computational efficiency, 
the so-called \emph{curse of dimensionality} must be overcome.
The relevance of operator learning stems from the fact that
numerous applications from engineering and the sciences require
large numbers of simulations of one physical model (typically a 
differential or integral equation) 
on families of homeomorphic, but distinct geometries. 
Typical examples are 
shape optimizations, 
shape identification, 
industrial design \cite{padula2023generativemodelsdeformationindustrial}
and, recently, 
also digital twins \cite{liu2024deepneuraloperatorenabled}.
Accordingly, 
recent years have seen significant 
algorithmic development of what could be termed 
``domain invariant'' operator networks. 
A (incomplete) list of references is 
\cite{cheng2024referenceneuraloperatorslearning,
duprez2025phifemfnonewapproachtrain,
goswami2022physicsinformeddeepneuraloperator,
DIVA,
LiHuangFNOGeo23,
li2023geometryinformedneuraloperatorlargescale,
loeffler2024graphfourierneuralkernels,
transolver,
taylor2025diffeomorphicneuraloperatorlearning,
yin2024dimonlearningsolutionoperators,
zeng2025pointcloudneuraloperator},
and further work mentioned in them.
These references focus on  ``computational'', algorithmic aspects,
i.e.\ they start from mathematical formulations 
and develop computational frameworks which allow, 
for suites of model problems, 
to realize 
the abovementioned generalization of domain-to-solution maps
across domain ensembles. 
See, for example, the discussion and examples in 
\cite{transolver,mousavi2025rignographbasedframeworkrobust,MIO,DistFunc,Zhao} 
and the references therein.

In this article, 
we propose a \emph{mathematical framework} for establishing
\emph{uniform approximation and generalization error bounds 
of neural and spectral operator networks for PDEs and BIEs} 
as in \cite{HSZ}
across ensembles of diffeomorphic domains and shapes.
Our framework is closely related conceptually to 
the recent ``DiMON'' in \cite{yin2024dimonlearningsolutionoperators}, 
and the ``Shape-DINO'' framework of \cite{gong2026shapederivativeinformedneuraloperators},
and references there.
We formulate a version of it and verify, based on our earlier work, 
applicability to a wide range of linear and nonlinear elliptic
and parabolic PDEs. A further novelty is
expression rate bounds for shape-to-solution neural operators
for \eqref{eq:model}.
This is based on our concept of 
{\em shape analyticity} or {\em shape holomorphy} 
of certain types of elliptic and parabolic PDEs \eqref{eq:model}
as well as of corresponding 
boundary integral equations \cite{AnDepBiOp,ShHolBIe,HenChS21},
developed by us and coauthors in recent years.
As we show here, the concept of \emph{shape analyticity} 
combined with recent resuts
\cite{HSZ} on expression-rate bounds for neural and spectral operator 
surrogates of holomorphic maps between function spaces,
will allow a mathematical analysis of expression rates.
Proposals for \emph{shape neural operators} have recently been 
developed e.g.\ in 
\cite{cheng2024referenceneuraloperatorslearning,
loeffler2024graphfourierneuralkernels,
yin2024dimonlearningsolutionoperators}
for (but not limited to) the expression of domain-to-solution maps
for well-posed operator equations
on parameter dependent domains $D_{{\bf y}}$, 
written in the generic form
\begin{equation}\label{eq:model}
%=================================
\mathrm{L}_{{\bf y}} u_{{\bf y}} = f_{{\bf y}} \text{ in } D_{{\bf y}}\,.
\end{equation}
For each parameter-instance ${\bf y}$, 
the (linear or nonlinear)
forward operator
$\mathrm{L} \in \calL_{\iso}(\calX_{\bf y} , \calY_{\bf y}')$ 
is assumed to be an isomorphism between suitable Banach 
spaces $\calX_{\bf y}$ and $\calY_{\bf y}'$, 
typically (tuples of) Sobolev spaces on the domain $D_{{\bf y}}$
(or on its boundary $\partial D_{{\bf y}}$), 
in which \eqref{eq:model} is well-posed.
For linear equations, \eqref{eq:model} corresponds to a
suitable variational formulation with parameter-dependent 
bilinear form $b_{\bf y}(\cdot,\cdot): \calX_{\bf y} \times \calY_{\bf y} \to \R$,
which satisfies \emph{uniform (with respect to $\bf y\in \square$) 
inf-sup conditions}.

Throughout, 
and similar to e.g.~\cite{yin2024dimonlearningsolutionoperators},
we assume the parameter ${\bf y}$ 
to have the structure
${\bf y}\in\square\isdef[-1,1]^\N$.
The sequence ${\bf y} \in \square$ 
is a (possibly high-dimensional) 
sequence of latent variables $(y_j)_{j\in\N}$
which encodes particular shapes and 
parametrizes the set $\bcalS$ of admissible shapes $D_{{\bf y}}$, 
i.e.
\begin{equation}\label{eq:S}
\bcalS = \{ D_{{\bf y}}: {\bf y} \in \square \} \,.
\end{equation}
Up to this point, \emph{no parametrization of $D_{{\bf y}}$ has been specified}.
One possible, diffeomorphic shape-parametrization 
will be specified in Assumption~\ref{ass:V}
in Section~\ref{sec:PrbFrm} ahead.

The abstract operator equation \eqref{eq:model} 
is subject to 
initial- and/or boundary conditions on $\partial D_{{\bf y}}$, 
with data collected in $f_{{\bf y}}$
so that a well-posed operator equation \eqref{eq:model} results
(examples are given ahead).
Well-posedness implies that 
for each instance of the data 
(which data includes also the domain $D\in \bcalS$),
there exists unique solution $u$ of \eqref{eq:model}.
We quantify the stability of \eqref{eq:model}
by assuming that there exists a constant $C_{\text{Stab}}$ 
such that
\[
\sup_{{\bf y}\in \square} \| \mathrm{L}^{-1} \|_{\calL( \calY_{\bf y}, \calX_{\bf y})} 
\leq 
C_{\text{Stab}}
\,. 
\]
The collections of solution- and data-spaces, 
$\calX_{\bf y}$ and $\calY_{\bf y}$, respectively,
for the formulation of the parametric model equation \eqref{eq:model} for ${\bf y}\in \square$
are denoted as
$$
\bcalX := \{ \calX_{\bf y} : {\bf y}\in \square\}\,, 
\quad 
\bcalY := \{ \calY_{\bf y} : {\bf y}\in \square\}\,.
$$
Typical choices are $\calX_{\bf y} = H^1_0(D_{\bf y})$,
e.g.\ for the Poisson equation in a parametric domain $D_{\bf y}$.
The corresponding collection of (parametric) 
shape-to-solution maps is denoted as
\[
\bcalG:     \bcalS\times \bcalY \to \bcalX: (D,f) \mapsto u \,,
\]
which is to say that
\[
\forall {\bf y}\in \square: \;\;
\calG_{{\bf y}}: \calS_{{\bf y}} \times \calY_{{\bf y}} \to \calX_{{\bf y}} : 
                 (D_{{\bf y}}, f_{{\bf y}}) \mapsto u_{{\bf y}}\,.
\]
For numerical stability of operator learning and training, 
we require
\emph{uniform well-posedness of the problems \eqref{eq:model}
      for all instances $D \in \bcalS$} of shapes.
Our aim is to construct for a given budget $N$ of NN parameters, 
sequences $ \{ \widetilde{\calG}_{N} \}_{N\in\N}$ 
of $N$-parametric surrogate solution operators  $\widetilde{\calG}_N$ 
together with bounds on the \emph{worst-case error} 
over the set $\bcalS$ of admissible shapes \eqref{eq:S}:
\[
\sup_{(D,f)\in \bcalS\times \bcalY} \big\| \calG(f) - \widetilde{\calG}_{N}(f)\big\|_{\calX}
=
\sup_{{\bf y}\in \square} 
\big\| \calG_{{\bf y}}(f_{{\bf y}}) - \widetilde{\calG}_{{\bf y},N}(f_{{\bf y}})\big\|_{\calX_{\bf y}}.
\]

We motivate the key ideas on a model linear elliptic 
divergence-form PDE \eqref{eq:model} in a parametric family of domains. 
Subsequently, 
we discuss a suite of possibly nonlinear, elliptic and 
parabolic PDEs which are covered by our abstract hypotheses.
Specifically, stationary Navier-Stokes equations, 
linear and nonlinear elastostatics,
boundary integral equations in acoustic scattering, 
linear parabolic IBVPs.
%%%%%%%%%%%%%%%%%%%%%%%%%%%%%%%%%%%%%%%%%%%%
\subsection{Contributions}
\label{sec:contr}
%%%%%%%%%%%%%%%%%%%%%%%%%%%%%%%%%%%%%%%%%%%%
We leverage \emph{shape-holomorphy} of solutions of elliptic and parabolic PDEs,
i.e.\ the holomorphic dependence of solutions on the shape of the physical domain 
in which the PDE (initial-)boundary value problem is to be solved.
For affine-parametric families of admissible shapes, 
we prove expression rate bounds of certain families of 
neural and spectral shape-to-solution surrogates.
Our contributions are as follows:

\noindent
1.\ 
We place ourselves in the mathematical operator learning framework of \cite{HSZ}.
This framework requires \emph{analytic data-to-solution maps} and
specifies particular (neural and spectral) ONet architectures, 
with PCA encoders and frame-decoders and 
convergence rate bounds in an abstract setting.
We verify, for a collection of representative PDEs from science and engineering, 
the assumptions underlying \cite{HSZ} which, in turn, imply existence of 
forward ONet architectures for a wide range of encoders and decoders (wavelets, 
PCAs, orthogonal polynomials, etc.).
Analyticity of the shape-to-solution map has, in recent years, been verified
for a wide range of linear and nonlinear differential- and integral equations.
We mention only 
\cite{AnDepBiOp,ShHolBIe,JDFH24,HenChS21} for boundary integral operators,
and 
\cite{Chernov,CSZ,ana,HSS,Valent} and the references there for PDEs.
As we show herein, for all these examples
\emph{neural operators} enable shape-transfer learning for PDEs, 
i.e.\ the generalization of PDE solutions for one ``nominal'' shape
to another, ``neighboring'' shape, without further training, as 
realized recently e.g.\ in \cite{transolver}.

\noindent
2.\  
We verify the analyticity hypotheses for parabolic evolution problems,
with shape-dependent linear principal part as considered in \cite{ana}, 
generating a $C_0$-semi\-group,
with possibly a nonlinear lower order term, in the classical framework 
(e.g.\ \cite{Pazy}), which was recently also studied in \cite{ana,HSS}.
We comment also in the possibility of long-term evolution prediction
in this setting.
%%%%%%%%%%%%%%%%%%%%%%%%%%%%%%%%%%%%%%%%%%%%
\subsection{Layout}
\label{sec:layout}
In Section~\ref{sec:Problemform}, 
we formalize the shape encoding and domain parametri\-zation 
upon which all subsequent results will be based.
Further, 
we introduce a suite of particular instances of \eqref{eq:model}
of possibly nonlinear elliptic and parabolic PDEs where we verify our abstract
hypotheses, and to which our results apply.
We state for each instance the relevant result on shape-holomorphy of 
parametric solution families.

In Section~\ref{sec:NeurSpcOpSr}, we introduce from \cite{HSZ} the 
main convergence rate bounds for neural and spectral operators
for holomorphic maps between Hilbert spaces, adapted to the present 
setting. 
We provide detailed discussions of shape encoders
and solution decoders which are admissible in our theory.

Section~\ref{sec:Conclusion} summarizes the main results, 
and indicates directions for further research, as well as 
references with numerical results.
%%%%%%%%%%%%%%%%%%%%%%%%%%%%%%%%%%%%%%%%%%%%
\subsection{Notation}
\label{sec:Notat}
Throughout the article, in order to avoid the repeated use of generic but 
unspecified constants, by \(C\lesssim D\) we mean that \(C\) can be bounded 
by a multiple of \(D\), independently of parameters which \(C\) and \(D\) may 
depend on. Obviously, \(C\gtrsim D\) is defined as \(D\lesssim C\) and 
\(C\sim D\) as \(C\lesssim D\) and \(C\gtrsim D\). For two Banach 
spaces $X$ and $Y$, we shall denote by $\calL(X,Y)$ the Banach 
space of bounded, linear operators $\mathrm{L}: X\to Y$.

For a bounded Lipschitz domain $D\subset  \R^n$, $n\geq 2$
(which is to say that the boundary $\partial D$ is locally a Lipschitz graph,
 \cite[Chpt.~1]{Grisvard85}), 
we introduce the set $\Diff(\overline{D})$ of diffeomorphisms from 
the closure $\overline{D}$ to $\R^n$. I.e., 
${\bf V}:\overline{D}\to \R^n$ belongs to $\Diff(\overline{D})$
if ${\bf V}\in C^1(\overline{D};\R^n)$ (continuously differentiable up to the boundary 
$\partial D$ and bijective from $\overline{D}$ to ${\bf V}(\overline{D})$).
When the Lipschitz domain $D\subset \R^n$ is connected,
we denote by ${\bf n}:\partial D \to \R^n$ the outward unit normal vector.
For $k\in \N_0$, $C^k(\overline{D};\R^n)$ denotes the space of $k$-times 
differentiable vector fields ${\bf V}$ 
whose derivatives are continuous up to $\partial D$.

By $\N_0^{\N}$ we denote the set of all sequences of 
non-negative inters. 
With $\calF$ we denote its subset of 
of sequences $\balpha = (\alpha_j)_{j\in\N}$ of finite total degree
$|\balpha| := \sum_{j\in\N} \alpha_j$, i.e.\ 
$\calF = \{ \balpha \in \N_0^{\N} : |\balpha| <\infty \}$.

%==========================================================
\section{Differential and integral equations on varying domains}
\label{sec:Problemform}
%==========================================================
Formulation of PDEs on families of domains
is well-established e.g. in the domain of shape-optimization
(e.g.~\cite{DelZol,SokZol} and references there).
Rather than focusing on evaluation of 
shape differentials which are key
in numerical shape optimization, however, 
we aim at mathematical statements on 
regularity and approximation rates of the domain-to-solution map. 
While being central to e.g.~shape uncertainty quantification 
and 
Bayesian shape identification (e.g.~\cite{HPS,vHScarab24,xiu1} and references there),
our purpose here is to infer expression rate and generalization error bounds for 
shape-to-solution learning for PDEs by deep neural operators.
To this end, 
we place ourselves in the mathematical operator learning framework of \cite{HSZ}.
This framework specifies particular (neural and spectral) ONet architectures, 
together with en- and decoders and convergence rate bounds in an abstract setting.
\emph{Domain parametrizations will be essential in this approach},
the parameters being latent variables in the learned shape-to-solution map.
%==========================================================
\subsection{Shape encoding by domain parametrization}
\label{sec:PrbFrm}
%==========================================================
Let $D_\refd\subset\mathbb{R}^d$ for \(d\in\N\) 
(of special interest are the cases \(d=2,3\)) denote a 
bounded ``reference'' or, in the parlance of operator learning,
latent domain with Lipschitz boundary \(\partial D_\refd\). 
We assume admissible shapes 
to be given via a \emph{parametric collection of diffeomorphisms} 
\[
\{ {\bf V}_{\bf y}: {\bf y} \in \square \} \subset \Diff(\overline{D_\refd};\R^n)\,.
\]
The family $\bcalS$ of admissible shapes is the collection of parametric domains 
\begin{equation}\label{eq:ParDom}
%====================================
\bcalS = \{ D_{{\bf y}}: {\bf y}\in \square \} \subset \R^n\,, 
\quad 
\mbox{where} 
\;\; 
D_{{\bf y}} \isdef{\bf V}_{{\bf y}}(D_\refd)\,, \quad {\bf y}\in \square \,.
\end{equation}
Specifically, 
for each parameter-instance ${\bf y} \in \square$,
${\bf V}_{{\bf y}} \in \Diff(\overline{D_\refd};\R^n)$ 
uniformly with respect to ${\bf y} \in \square$.
We consider sets $\square$ of 
scaled parameters ${\bf y}$ 
which are sequences 
\[
{\bf y}\in \square \isdef [-1,1]^{\N} \,.
\]
We endow $\square$ with the metric 
$d_\square(\cdot,\cdot): \square\times\square\to \R$
given by 
$$
\forall {\bf y},{\bf y}'\in \square: \quad 
d_\square({\bf y} , {\bf y}') := \sup_{k\in \N} |y_k - y'_k| 
\,.
$$
The Tikhonov theorem (see, e.g.,  \cite[page 143: Thm.\ 13]{Ke1955})
implies that $(\square, d_\square)$ is a compact metric space.
We endow $\square$ with a probability measure $\lambda = \lambda_1^{\otimes \infty}$,
i.e. with the countable product of the univariate 
Lebesgue probability measure $\frac{1}{2} \lambda_1$ on $[-1,1]$.

For every ${\bf y} \in \square$, ${\bf V}_{\bf y}\in \Diff(\overline{D_\refd};\R^n)$ 
is continuously differentiable with respect to ${\bf x}\in D_\refd$, up to $\partial D_\refd$.
We formalize this in the \emph{uniformity condition}
\begin{equation}\label{eq:uniformity}
%====================================
\sup_{{\bf y}\in \square}
\|{\bf V}_{{\bf y}}\|_{{C^1(\overline{D_{\bf y}};\mathbb{R}^d)}}
+
\|{\bf V}^{-1}_{{\bf y}}\|_{{C^1(\overline{D_{\bf y}};\mathbb{R}^d)}}
\leq C_\square
\,.
\end{equation}
The uniformity condition \eqref{eq:uniformity} 
ensures that there exist constants 
$0< \underline{\sigma} < \overline{\sigma}$
(depending on $C_\square$ in \eqref{eq:uniformity})
such that
for every ${\bf y}\in\square$ and for ${\bf x}\in\overline{D_\refd}$ 
the singular values of the Jacobian ${\bf J}_{{\bf y}}$ 
for the vector field \({\bf V}_{\bf y}\) 
satisfy
\begin{equation}\label{eq:Jbound}
%========================================
0<\underline{\sigma}
\leq
\min_{{\bf x}\in \overline{D_\refd}} \big\{\sigma\big({\bf J}_{{\bf y}}({\bf x})\big) \big\}
\leq
\max_{ {\bf x}\in \overline{D_\refd}} \big\{\sigma\big({\bf J}_{{\bf y}}({\bf x})\big) \big\}
\leq
\overline{\sigma}
<\infty\,.
\end{equation}
Throughout, 
we make the following assumption on affine-parametric encoder
structure of the vector field ${\bf V}$
through a weighted sum with respect to a 
sequence of `shape feature-vectors' 
$({\boldsymbol\varphi}_k)_{k\in\N}$.
\begin{assumption}\label{ass:V}
%====================================
The vector field ${\bf V}_{{\bf y}}$ 
admits an affine-parametric expansion 
\begin{equation}\label{eq:KL}
%=====================================
{\bf V}_{{\bf y}}({\bf x}) 
= 
{\bf V}_\bsnul({\bf x}) 
+ \sum_{k=1}^\infty w_k{\boldsymbol\varphi}_k({\bf x})y_k\,,
\quad {\bf x} \in D_\refd \,,
\end{equation}
with the parameter sequence
${\bf y} = (y_k)_{k\in\N} \in \square$,
shape-feature sequence
$({\boldsymbol\varphi}_k)_{k\in\N} \subset C^1(\overline{D_\refd};\R^n)$,
and weight sequence
$\bsw = (w_k)_{k\in\N}\in (0,\infty)^{\mathbb N}$ 
which is chosen such that for all $\eps > 0$ holds 
\begin{equation}\label{eq:wght}
\bsw^{1+\eps} := (w_k^{1+\eps})_{k\in\N} \in \ell^1({\mathbb N})\,.
\end{equation}
Then 
\begin{equation}\label{eq:gammak}
%=====================================
\bgamma 
=
(\gamma_k)_{k\in\N}
\isdef \big(w_k \|{\boldsymbol\varphi}_k\|_{C^{1}({{\overline{D_\refd}}};
	\mathbb{R}^{d})}\big)_{k\in\N}\in \ell^1(\N)\,.
\end{equation}
We denote the \(\ell^1\)-norm of this sequence 
by 
$ c_\bgamma\isdef\sum_{k=1}^\infty\gamma_k$.
\end{assumption}
The expansion \eqref{eq:KL} provides an affine-parametric
expansion of any admissible shape in $\bcalS$.
One important instance of \eqref{eq:KL} is 
orthonormality of the $({\boldsymbol\varphi}_k)_{k\in\N}$
in a suitable innerproduct $(\cdot,\cdot)_\bcalS$ 
on (an ambient space of) the admissible shapes $\bcalS$. 
Then, $y_k=( {\bf V}_{\bf y}, {\boldsymbol\varphi}_k)_{\bcalS}$
is the co-ordinate in e.g.~a PCA of $\bcalS$.
We develop a mathematical framework for
affine-parametric representations \eqref{eq:KL} of shapes 
in Sections~\ref{sec:ParEnc} and \ref{sec:BiOrth} ahead.

With the affine-parametric map \eqref{eq:KL}, we write
for all ${\bf y}\in \square$: ${\bf V}_{{\bf y}}: D_\refd \to D_{{\bf y}}$.
One refers to 
\[
{D}_{{\bf 0}} := {\bf V}_\bsnul(D_\refd)
\]
as ``nominal domain''. 
When ${\bf V}_{\bf 0} = {\bf Id}$, $D_{\bf 0} = D_\refd$.
Generally, 
$D_{\bf 0}$ may be realized as a physical domain, 
whereas
$D_\refd$ could be a ``latent'' reference domain
in the sense of continuum mechanics \cite{TruesdellV1}.
Note that the shape-encoding \eqref{eq:KL} is linear. 
Nonlinear encodings 
as e.g.\ in \cite{wen2025geometryawareoperatortransformer} 
could offer quantitative advantages.
\begin{remark}[Multi-resolution shape encoding]\label{rmk:MRAShape}
The shape-encoding in Assumption~\ref{ass:V} is in terms of 
``principal shape components'' ${\boldsymbol\varphi}_k$ 
which can be obtained for example as eigenbasis of 
empirical covariances from digitally acquired geometry snapshots. 
Other encoders with affine-parametric structure 
\eqref{eq:KL} are conceivable. 
\emph{Multi-resolution geometry representations} 
as introduced in Section~\ref{sec:Frames} ahead 
can be an alternative which avoids a numerical computation 
of the PCA and affords multiscale feature-resolution and 
geometry de-featuring by thresholding.
\end{remark}
We record properties of the affine-parametric shape representation \eqref{eq:KL}.
\begin{lemma}
Let ${\bf V}_\bsnul \in \Diff(\overline{D_\refd};\overline{D_{\bsnul}})$ 
be such that
\[
  \|{\bf V}_0\|_{{C^{1}(\overline{D_{\refd}};\mathbb{R}^d)}} 
  +
  \|{\bf V}_0^{-1}\|_{{C^{1}(\overline{D_{\bsnul}};\mathbb{R}^d)}}
  \le C\,.
\] 
Then, under Assumption~\ref{ass:V},
${\bf V}_{\bf y}$ 
for all ${\bf y}\in\square$ satisfy the uniformity bound \eqref{eq:uniformity},
provided that $c_\bgamma$ is sufficiently small.
\end{lemma}
\begin{proof}
All shapes ${\bf V}_{\bf y}$ are images of ${D}_{{\bf 0}}$ under a 
family of diffeomorphisms $\widetilde{\bf V}_{\bf y}$ which are isotopic
to the identity on ${D}_{{\bf 0}}$.
The assumption ${\bf V}_\bsnul \in \Diff(\overline{D_\refd};\overline{D_{\bsnul}})$
implies that ${D}_{{\bf 0}}$ is a bounded Lipschitz domain,
and 
\[
\widetilde{\bf V}_{{\bf y}}({\bf x}) 
= 
{\bf x} +
\sum_{k=1}^\infty w_k\widetilde{\boldsymbol\varphi}_k({\bf x})y_k\,,
	\quad {\bf x}\in {D}_{{\bf 0}}, \; {\bf y}\in \square \,.
\]
Here, for $k\in \N$,
$
\widetilde{\boldsymbol\varphi}_k = {\boldsymbol\varphi}_k\circ {\bf V}_\bsnul^{-1}
$
and
we set 
$\widetilde{\gamma}_k
:= w_k\|\widetilde{\boldsymbol\varphi}_k \|_{{C^1(\overline{D_\bsnul};\mathbb{R}^d)}}$. 
For any ${\bf y}\in \square$, 
the deformation field 
$\widetilde{\bf V}_{{\bf y}} := {\bf V}_{\bf y} \circ {\bf V}_\bsnul^{-1}  
                              : {D}_{{\bf 0}} \to D_{\bf y}$ 
is invertible since the parametric Jacobian
\[
\widetilde{\bf J}_{{\bf y}}({\bf x})
= 
{\bf I} +
\sum_{k=1}^\infty w_k\widetilde{\boldsymbol\varphi}_k'({\bf x})y_k\,,
	\quad {\bf x}\in {D}_{{\bf 0}}\,,
\]
is uniformly bounded
\[
0 
\le 1-\sum_{k=1}^\infty \widetilde{\gamma}_k 
\le \|\widetilde{\bf J}_{{\bf y}}\|_{C^0(\overline{D_{\bf 0}};\mathbb{R}^{d\times d})} 
\le 1+\sum_{k=1}^\infty \widetilde{\gamma}_k < \infty
\]
provided that 
\[
  c_{\widetilde{\boldsymbol\gamma}} := \sum_{k=1}^\infty \widetilde{\gamma}_k < 1\,.
\]
Since $\widetilde\gamma_k\le \gamma_k\|{\bf V}_\bsnul \|_{{C^1(\overline{D_\refd};\mathbb{R}^d)}}$
and hence 
$c_{\widetilde{\boldsymbol\gamma}}
 \le c_{\boldsymbol\gamma} \|{\bf V}_\bsnul\|_{C^1(\overline{D_\refd};\mathbb{R}^d)}$, 
the claim follows in view of
\[
  {\bf V}^{-1}_{{\bf y}} = {\bf V}_\bsnul^{-1} \circ \widetilde{\bf V}^{-1}_{{\bf y}}:
    D_{\bf y} \to D_\refd \,.
\]
\end{proof}

In order to ensure that the model problem \eqref{eq:model} is 
well-defined for almost every \({\bf y}\in\square\), we consider the 
${\bf x}$-dependent input data, like the coefficients of the differential 
operator $\mathcal{L}$ and the right-hand side $f$, 
to be defined on the \emph{hold-all} domain $\mathcal{D}$, 
\begin{equation}\label{eq:holdall}
%=======================================
\overline{\bigcup_{{\bf y}\in\square}D_{\bf y}} \subseteq \mathcal{D} \subset \R^n\,.
\end{equation}
As in e.g.~\cite{yin2024dimonlearningsolutionoperators}, we
analyse the smoothness of parametric solution families of \eqref{eq:model},
$u_{\bf y} \in D_{\bf y}$ with respect to the parameter ${\bf y}$
via the pull-back of the partial differential equation \eqref{eq:model} 
and its parametric solution $u_{\bf y}$ to the reference domain $D_{\refd}$. 
Then, for linear elliptic or parabolic PDEs 
one obtains for the parametric solution 
$\hat{u}_{\bf y} = (u\circ {\bf V}_{\bf y}^{-1}): 
D_{\refd}\to \mathbb{R}$ an analytic estimate of the following type: 
\begin{equation}\label{eq:regu}
%=====================================
\forall {\bf y} \in \square\; 
\forall \balpha \in \calF: \quad 
\big\|\partial^\balpha_{\bf y}\hat{u}_{\bf y} \big\|_{H^1(D_\refd)}
\lesssim |\balpha|!\rho^{|\balpha|}\bgamma^\balpha 
\end{equation}
Here, $\rho > 0$ is an appropriate constant and $V$ 
denotes a Hilbert space on $D_{\refd}$. 
The roadmap for the verification of such \emph{growth-bounds}
comprises roughly the following three steps:
\begin{enumerate}
\item
The specific, affine structure of the domain deformation 
field \eqref{eq:KL} implies that the parameter-to-domain 
map is analytic.
\item
One verifies that the domain-to-data map is analytic. 
\item
In $D_\refd$,
the data-to-solution map has also to be analytic.
The verification of analyticity 
depends on the particular boundary value 
problem under consideration. 
Below, we indicate a wide range
of linear and certain nonlinear PDEs where this holds, 
and which are hence within the present framework.
\end{enumerate}
As shown in \cite[Sct.~2.4]{HSS}, 
then the parameter-to-solution map,
which is given as the composition of these three maps,
is also analytic and satisfies \eqref{eq:regu}.

The analytic estimate \eqref{eq:regu} for 
the parameter-dependence is the key estimate for 
deriving approximation rate results for 
gpc, QMC, sparse-grid and neural approximation methods,
for the solution to \eqref{eq:model} 
as it gives a guideline how to 
construct appropriate approximation spaces with
respect to the high-dimensional parameter ${\bf y}$. 
%========================================
\subsection{Examples of operator equations}
\label{sec:Expls}
%========================================
In the following, we present several examples 
and the respective shape-holomorphy 
results for different applications.
%=======================================
\subsubsection{Poisson equation}
\label{sec:Poisson}
%=======================================
As a standard instance of \eqref{eq:model},
we consider diffusion in 
a homogeneous, isotropic medium, i.e.\ the Poisson equation,
occupying a domain $D_{\bf y}$ 
with homogeneous Dirichlet boundary conditions
\begin{equation}\label{eq:Poisson}
%======================================
  -\Delta u_{\bf y} = f\ \text{in}\ D_{\bf y}\,, 
  \quad u_{\bf y} = 0\ \text{on}\ \partial D_{\bf y}\,.
\end{equation}
This problem has been considered first in \cite{xiu2,xiu1}
and analyzed in \cite{CaNT,HPS,H3S}. The main tool used 
in the analysis for the model problem \eqref{eq:Poisson} 
is the one-to-one correspondence between the problem 
on the actual realization $D_{\bf y}$ and its pull back 
to the reference domain $D_\refd$. 
The equivalence between those two problems is described by the vector 
field ${\bf V}_{\bf y} : D_\refd \to D_{\bf y} = {\bf V}_{\bf y}(D_\refd)$
as in \eqref{eq:KL}.

In the sequel, 
we choose the Sobolev spaces \( \calX_{\bf y} = H^1_0\big(D_{\bf y} \big)\) 
and \( V = H^1_0(D_\refd)\),
equipped with the norms 
\(
\|\cdot\|_{H^1(D_{\bf y})}\isdef\|\nabla\cdot\|_{L^2(D_{\bf y};\mathbb{R}^d)}
\text{ and }
\|\cdot\|_{H^1(D_\refd)}\isdef\|\nabla\cdot\|_{L^2(D_\refd;\mathbb{R}^d)},
\)
respectively. 
Then, 
for given \({\bf y}\in\square\) and for ${\bf V}_{\bf y}$ as in \eqref{eq:KL}, 
the variational formulation 
for the model problem \eqref{eq:Poisson} 
reads as follows: 
Find \(u_{\bf y} \in H^1_0\big(D_{\bf y} \big)\) 
such that
\[
\int_{D_{\bf y} }\langle\nabla u_{\bf y},\nabla v\rangle\d{\bf x}
=\int_{D_{\bf y}} fv\d{\bf x}\quad\text{for all }v\in H^1_0\big(D_{\bf y} \big)\,.
\]
Thus, with 
\begin{equation}\label{eq:ParCoef}
{\bf A}^{\refd}_{\bf y}({\bf x})
\isdef\big({\bf J}_{\bf y}({\bf x})^\intercal{\bf J}_{\bf y}({\bf x})\big)^{-1}
\det{\bf J}_{\bf y}({\bf x})\,,\quad {\bf x}\in D_\refd \,,
\end{equation}
and
\begin{equation}\label{eq:ParRHS}
f^{\refd}_{\bf y}({\bf x})\isdef (f\circ{\bf V}_{\bf y} )({\bf x})\det{\bf J}_{\bf y}({\bf x})\,,
\quad {\bf x}\in D_\refd \,,
\end{equation}
we obtain the following 
\emph{pulled-back variational formulation}
with respect to the reference domain: 
For ${\bf y} \in \square$,
find $\hat{u}_{\bf y} \isdef u_{\bf y}\circ{\bf V}^{-1}_{\bf y}\in H^1_0(D_\refd)$
such that 
\begin{equation}\label{eq:varform}
%=====================================
\int_{D_\refd}\langle {\bf A}^{\refd}_{\bf y}\hat{\nabla}\hat{u}_{\bf y},\hat{\nabla} v\rangle\d{\bf x}
=
\int_{D_\refd} f^{\refd}_{\bf y} v \d{\bf x}\quad
\text{for all } v \in H^1_0(D_\refd)\,.
\end{equation}
Here and in the following, $\hat{\nabla}$ denotes the 
gradient with respect to the reference coordinates $\hat{{\bf x}}\in D_\refd$, and
\(\langle\cdot,\cdot\rangle\) the canonical inner product for \(\mathbb{R}^d\). 
In view of Assumption~\ref{ass:V}, this problem is stable and 
uniquely solvable for all parameter values ${\bf y}\in\square$.

To derive from \eqref{eq:varform} regularity results for 
$\hat{u}_{\bf y} \in H_0^1(D_{\refd})$ with respect to the 
parameter ${\bf y}\in\square$, 
one requires assumptions on the regularity 
of the right-hand side \(f^{\refd}_{\bf y}\in H^{-1}(D_{\refd})\)
with respect to this parameter. 
Hence, we have to assume that \(f\) is a smooth function. 
Under this condition, one has 
in accordance with \cite[Thm.~5]{HPS} the following result:

\begin{theorem}\label{thm:regu}
%=====================================
Let $f\in C^\infty(\mathcal{D})$ be an analytic function 
on the hold-all $\mathcal{D}$ defined in \eqref{eq:holdall}.
Then, the derivatives of the solution \(u\) to the variational 
problem \eqref{eq:varform} satisfy \eqref{eq:regu}
for some constant $\rho>0$.
\end{theorem}

\begin{proof}
Let us sketch an argument to obtain statements like Theorem~\ref{thm:regu}. 
To this end, consider the Hilbert space $V=H^1_0(D_\refd)$ and its dual 
$V' = H^{-1}(D_\refd)$. 
Also, for a symmetric, positive definite coefficient matrix
${\bf A}\in L^\infty(D_\refd; \mathbb{R}^{d\times d})$ 
and some constants $0<a_- \leq a_+ <\infty$, 
we define the set of admissible coefficient functions by
\begin{equation}\label{eq:Da+a-}
\begin{aligned}
\calD(a_-,a_+)
&:= \big\{ {\bf A} \in L^\infty(D_\refd, \mathbb{R}^{d\times d}_{\sym}) 
\mid \\
&\qquad\qquad\forall \boldsymbol\xi \in \mathbb{R}^d: 
a_- \|\boldsymbol\xi\|^2 
\leq \boldsymbol\xi^\intercal {\bf A}({\bf x}) \boldsymbol\xi 
\leq a_+ \|\boldsymbol\xi\|^2 
\big\}\,.
\end{aligned}
\end{equation}

\emph{(i.) Analyticity of the parameter-to-data map.} 
For affine-parametric ${\bf V}_{\bf y}({\bf x})$ as in \eqref{eq:KL},
the Jacobian matrices ${\bf J}_{\bf y}({\bf x})$ and their transposes 
${\bf J}_{\bf y}({\bf x})^\intercal$ are likewise affine-parametric, hence 
depend analytically on the components $y_j$ of ${\bf y} \in \square$. 
According to \eqref{eq:Jbound}, 
they are also uniformly with respect to ${\bf y}$ nonsingular. 
This implies that both, 
\[
\big({\bf J}_{\bf y}(\cdot)^\intercal{\bf J}_{\bf y}(\cdot)\big)^{-1}: 
\square \to L^\infty(D_\refd; \mathbb{R}^{d\times d})
\ \mbox{and}\ 
\det{\bf J}_{\bf y}(\cdot): \square \to L^\infty(D_\refd)\,,
\]
depend analytically on ${\bf y}$ (being linear combinations 
of products of elements of ${\bf J}_{\bf y}({\bf x})$), 
uniformly with respect to ${\bf y}\in \square$. 
According to \eqref{eq:ParCoef}, 
there exist constants $0<a_- \leq a_+ <\infty$ (depending on 
$\underline{\sigma}, \overline{\sigma}$ and on the constant 
$C>0$ in \eqref{eq:uniformity}) such that
\[
{\bf A}: \square \to \calD(a_-,a_+) \subset L^\infty(D_\refd; \mathbb{R}^{d\times d}): 
{\bf y}\mapsto \big({\bf J}_{\bf y}(\cdot)^\intercal{\bf J}_{\bf y}(\cdot)\big)^{-1}
\det{\bf J}_{\bf y}(\cdot)
\]
is analytic. 
Arguing similarly, we see also that ${\bf y}\mapsto 
f^{\refd}_{\bf y}(\cdot)$, as defined in \eqref{eq:ParRHS}, 
is analytic as a map $\square\to V'$. 
Hence, the parameter-to-data map 
\begin{equation}\label{eq:para2data}
%====================================
  \square\to\calD(a_-,a_+)\times V',\quad
  {\bf y}\mapsto ({\bf A}_{\bf y}^{\refd},f_{\bf y}^\refd)
\end{equation}
is analytic.

\emph{(ii.) Analyticity of the data-to-solution map.} 
For ${\bf A}\in\calD(a_-,a_+)$, the differential operator 
$\mathrm{L}({\bf A} , \partial_{{\bf x}}) := -\div ({\bf A} \nabla) \in \calL(V,V')$ 
is an isomorphism by the Lax-Milgram lemma. 
It therefore admits a bounded inverse $\mathrm{L}^{-1}\in \calL_{\iso}(V',V)$. 
The inversion map 
\[
\mathrm{inv}: \calL_{\iso}(V,V') \to \calL_{\iso}(V',V): \mathrm{L} \mapsto \mathrm{L}^{-1} 
\]
is analytic, which is easily verified by a Neumann series argument.
Since
\[
{\mathbb R}^{d\times d}_{\sym} \to \calL(V,V'): 
{\bf A} \mapsto \mathrm{L}({\bf A} , \partial_{{\bf x}}) 
\]
is linear, it is in particular analytic. Accordingly, the compositional map
\[
\calD(a_-,a_+) \to  \calL_{\iso}(V',V): 
{\bf A} \mapsto \mathrm{L}({\bf A},\partial_{{\bf x}})^{-1}
\]
is analytic. Since moreover the map ${\mathrm{apply}}$ 
given by 
\[
{\mathrm{apply}}: \calL_{\iso}(V',V) \times V' \to V : 
(\mathrm{L}^{-1} , f) \mapsto  \mathrm{L}^{-1}f
\]
is bilinear, it is also hence analytic, which implies that 
the data-to-solution map
\[
\calD(a_-,a_+)\times V' \to V: ({\bf A},f) \mapsto 
\mathrm{apply}\big(\mathrm{L}({\bf A},\partial_{\bf x})^{-1}, f\big) 
\]
is analytic, as a composition of analytic maps.

\emph{(iii.) Analyticity of the parameter-to-solution map.}
In view of the analyticity of the parameter-to-data map from 
item \emph{(i.)} and the analyticity of the data-to-solution 
map from item \emph{(ii.)}, we conclude that the composed 
map, that is the parameter-to-solution map given by
\[
\square \to V: {\bf y} \mapsto \mathrm{apply}
\big(\mathrm{L}({\bf A}_{\bf y},\partial_{\bf x})^{-1},f^{\refd}_{\bf y}\big)
\]
is analytic. 
Especially, the decay of $\partial_{y_k}$ with respect of
the coordinate dimension $k$ is inherited, which gives the desired 
bound \eqref{eq:regu} of the solution's derivates, see 
e.g.~\cite[Sct.~2.4]{HSS} for the details.
\end{proof}
\begin{remark}\label{rmk:NonAff}
The analyticity of affine-parametric dependence \eqref{eq:KL} of the data-to-domain 
map can be replaced by analyticity of any non-affine dependence, e.g.~the so-called 
``log-affine'' dependence where an exponential is inserted into the parametrization.
\end{remark}

\begin{remark}\label{rmk:Gs}
The proof of Theorem~\ref{thm:regu} exploits,
in an essential fashion, that compositions
of analytic maps are again analytic.
This reasoning can also be performed in a more general setting, if for example
some of the data depend only Gevrey-regular on parameters.
Using the Gevrey composition lemma from \cite[Sct.~2.4]{HSS}, 
corresponding results on Gevrey-regular dependence similar to the recent results in \cite{Chernov,ana}
follow by the exact same, qualitative arguments.
Detailed bootstrap proofs via induction and the chain rule as in \cite{Chernov,ana,HPS}
can provide, however, sharper bounds on the constant $\rho$ in \eqref{eq:regu}.
\end{remark}

%=======================================
\subsubsection{Heterogeneous media. Eulerian and Lagrangian perspective}
\label{sec:interfaces}
%=======================================
The discussion in the previous section refers 
to the Poisson equation \eqref{eq:Poisson}, 
i.e.\ to diffusion in homogeneous media. 
In heterogeneous media, where
stationary diffusion is governed by
the more general diffusion equation
\begin{equation}\label{eq:diffusion_euler}
%======================================
  -\div\big({\bf A}\nabla u_{\bf y} \big) = f\ \text{in}\ D_{\bf y}\,, 
  \quad u_{\bf y} = 0\ \text{on}\ \partial D_{\bf y}\,,
\end{equation}
one obtains, in contrast to \eqref{eq:ParCoef}, 
\[
{\bf A}^{\refd}_{\bf y}({\bf x}) 
= 
{\bf J}_{\bf y}({\bf x})^{-\intercal} ({\bf A}\circ {\bf V}_{\bf y})({\bf x})
{\bf J}_{\bf y}({\bf x})^{\intercal}\det{\bf J}_{\bf y}({\bf x})
\]
in the variational formulation \eqref{eq:varform}. 
We see 
that the domain-to-data map is analytical in this case 
\emph{whenever the diffusion matrix ${\bf A}$ is an analytic function of
${\bf x}$ in the hold-all $\mathcal{D}$}. 
Theorem~\ref{thm:regu} holds also true in this situation
provided that ${\bf A}$ is in addition uniformly elliptic on $\mathcal{D}$.
Similarly, if ${\bf x}\mapsto {\bf A}({\bf x})$ is $s$-Gevrey regular
in ${\mathcal D}$, the domain-to-data map is $s$-Gevrey regular,
by the Gevrey-composition Lemma in \cite[Sct.~2.4]{HSS}.

The situation becomes different if we consider Lagrangian
coordinates for the diffusion matrix under consideration.
This means that \eqref{eq:diffusion_euler} is given as
\begin{equation}\label{eq:diffusion_lagrange}
%======================================
  -\div\big(({\bf A}\circ{\bf V}_{\bf y}^{-1}) \nabla u_{\bf y} \big) = f\ \text{in}\ D_{\bf y}\,, 
  \quad u_{\bf y} = 0\ \text{on}\ \partial D_{\bf y}\,,
\end{equation}
which leads to 
\[
{\bf A}^{\refd}_{\bf y}({\bf x}) 
= 
{\bf J}_{\bf y}({\bf x})^{-\intercal} {\bf A}({\bf x}) {\bf J}_{\bf y}({\bf x})^{\intercal}
\det{\bf J}_{\bf y}({\bf x})\,,
\quad
{\bf x} \in \domain 
\,.
\]
In this case, the domain-to-data map is analytical also 
for ${\bf A}$ being just in $L^\infty(D_{\refd},\mathbb{R}^{d\times d}_{\sym})$.
Therefore, the decay estimate \eqref{eq:regu} is valid. 
Since
the diffusion matrix ${\bf A}$ in \ref{eq:diffusion_lagrange} 
can also be piecewise smooth, e.g.
\[
  {\bf A}({\bf x}) = \sum_{i=1}^M \alpha_i \mathds{1}_{S_{\refd,i}({\bf x})}
\]
with $\alpha_i\in\mathbb{R}$ and $\mathds{1}_{S_{\refd,i}}$ 
being the indicator functions of subdomains $S_{\refd,i}
\subset D_\refd$, interface problems are covered in this way, 
see \cite{HPS} for details.

\begin{remark}\label{rmk:SemiLin}
The PDEs \eqref{eq:Poisson} and
\eqref{eq:diffusion_euler} considered so far were linear and elliptic. 
An extension of the preceeding arguments to semilinear, 
elliptic PDEs with Gevrey nonlinearity is available in \cite{HSS}.
Eigenvalue problems with parametric Gevrey-dependence 
have been considered in \cite{ChernEVPGevr}, albeit without consideration
of domain-variations. 
Using the present formalism, Gevrey-regular dependence
of the first eigenvalue on the shape of the domain 
could be shown also for the setting in \cite{ChernEVPGevr}.
\end{remark}

%=======================================
\subsubsection{Linear elasticity}
\label{sec:LinElas}
%=======================================
The equations of so-called ``linear elastostatics''
constitute a mathematical model for small deformations
and internal stress of solid objects in equilibrium,
subject to certain load conditions.
Being linear, elliptic divergence-form PDEs, these equations 
are structurally similar to \eqref{eq:diffusion_euler}, and 
corresponding results on analytic dependence of displacement
fields on the physical configuration hold.

For the mathematical formulation, we partition the 
boundary $\Gamma_\refd = \partial D_\refd$ into two 
disjoint, open parts $\Gamma^D_{\refd}$ and $\Gamma^N_{\refd}$
such that 
\[
|\Gamma^D_{\refd}| > 0\,, \quad 
\overline{\Gamma}_{D,\refd}\cup\overline{\Gamma}_{N,\refd}=\Gamma_\refd\,,
\]
on which we shall pose Dirichlet boundary conditions 
or Neumann boundary conditions, respectively. 
Moreover,
we set
\[
  [H_D^1(D_\refd)]^d \isdef \big\{{\bf u}\in [H^1(D_\refd)]^d: 
  	{\bf u} = {\bf 0}\ \text{on}\ \Gamma^D_{\refd}\big\}
\]
and denote its transported counterpart by
 \[
  \big[H_D^1\big(D_{\bf y} \big)\big]^d \isdef \big\{{\bf u}_{\bf y}\in \big[H^1\big(D_{\bf y} \big)\big]^d: 
  	{\bf u}_{\bf y} = {\bf 0}\ \text{on}\ \Gamma^D_{\bf y}\big\}\,,
\]
where 
$\Gamma^D_{\bf y} \isdef {\bf V}_{\bf y}(\Gamma^D_{\refd})$
and likewise 
$\Gamma^N_{\bf y} \isdef {\bf V}_{\bf y}(\Gamma^N_{\refd})$. 

We search for the displacement field
${\bf u}_{\bf y}\in \big[H_D^1\big(D_{\bf y} \big)\big]^d$ 
of an elastic body made of homogeneous, isotropic material, and
occupying $D_{\bf y}$.  
To that end, we use the following 
notation for the respective linearized strain tensor
\[
	{\bf u}_{\bf y} \mapsto \varepsilon({\bf u}_{\bf y})\isdef\frac{1}{2}\big(\nabla {\bf u}_{\bf y}+(\nabla{\bf u}_{\bf y})^\intercal\big)\,.
\]
It is linked to the stress tensor $\sigma({\bf u})$ 
by Hooke’s law:
\[
  \sigma({\bf u}_{\bf y}) = \frac{E\nu}{(1+\nu)(1-2\nu)}\trace\big(\varepsilon({\bf u}_{\bf y})\big) {\bf I}
	+ \frac{E}{1+\nu} \varepsilon({\bf u}_{\bf y})\,.
\]
Here, $E > 0$ is the Young modulus and $0<\nu<1/2$
is the Poisson ratio. In the equilibrium, 
we then have
\begin{equation}\label{eq:elasticity}
%====================================
\left\{\;\begin{aligned}
-\div\big(\sigma({\bf u}_{\bf y}) \big) &= {\bf f}&&\text{in}\ D_{\bf y}\,,
\\[1ex]
{\bf u}_{\bf y} &= \mathbf{0}&&\text{on}\ \Gamma^D_{\bf y}\,,
\\[1ex]
\sigma\big({\bf u}_{\bf y} \big){\bf n} &= {\bf g}&&\text{on}\ \Gamma^N_{\bf y}\,,
\end{aligned}\right.
\end{equation}
where ${\bf n}$ denotes the unit exterior normal at the 
parametric domain's boundary $\Gamma_{\bf y} = \partial D_{\bf y}$. 
This formulation of 
so-called ``linear elasticity'' 
is known as the displacement or primal formulation.
By transforming the variational formulation in $D_{\bf y}$ onto
the reference domain $D_{\refd}$, one obtains the following
result on the regularity of the solution to \eqref{eq:elasticity},
see \cite{HKS} for details. 
We recall the hold-all domain $\mathcal{D}$ from \eqref{eq:holdall}.

\begin{theorem}[{\cite[Thm.~3.4]{HKS}}]\label{thm:elasticity}
%==============================================
Let ${\bf u}_{\bf y} \in \big[H_D^1\big(D_{\bf y} \big)\big]^d$ 
be the solution to \eqref{eq:elasticity} 
and let 
$\widehat{\bf u}_{\bf y} \isdef {\bf u}_{\bf y}\circ{\bf V}^{-1}_{\bf y}\in [H_D^1(D_\refd)]^d$ 
denote the respective pull-back onto $D_{\refd}$. 
If the body force
${\bf f}\in [C^\infty(\mathcal{D})]^d$ 
and the surface force 
${\bf g}\in [C^\infty(\mathcal{D})]^d$
are analytic in $\mathcal{D}$ and the deformation 
${\bf V}_{\bf y}$ satisfies Assumption~\ref{ass:V}, 
then there holds
\[
\forall {\balpha}\in\calF: 
\quad 
  \|\partial^{\balpha}_{{\bf y}}\hat{\bf u}_{\bf y}\|_{[H^1(D_\refd)]^d} 
  	\lesssim |\balpha|! \rho^{|\balpha|}\bgamma^{\balpha}
\,.
\]
\end{theorem}

This theorem remains valid also for materials with more 
general properties, where the stress tensor is given by 
means of a fourth order tensor in accordance with 
$\sigma({\bf u}) \isdef \mathbb{C} : \varepsilon({\bf u})$.
Especially, in accordance with Subsection~\ref{sec:interfaces},
composite materials can be modeled by means of the
material tensor
\[
 \mathbb C({\bf x},{\bf y}) = \mathbb{C}_0 
 	+ \sum_{i=1}^N\mathbb C_i \mathds{1}_{S_{{\bf y},i}}({\bf x})\,,
\]
where $\mathds{1}_{S_{{\bf y},i}}$ denote the indicator 
functions for $S_{{\bf y},i}\isdef S_{\refd,i}\circ{\bf V}_{\bf y}$
of subdomains $S_{\refd,i}\subset D_{\refd}$ and $\mathbb{C}_i$ 
denote constant, symmetric fourth order tensors. We 
refer to \cite{HKS} for the details.

These considerations pertained to so-called ``linear elasticity'',
modelling elastic bodies underoing small deformations.
Corresponding results are expected to hold also for large deformations,
under suitable constitutive laws, where analytic dependence of 
deformations on problem parameters is known \cite{Valent} to hold.
%=======================================
\subsubsection{Stationary Navier-Stokes equations}
%=======================================
We consider next the stationary Navier-Stokes equations, 
in an ensemble $\{ D_{\bf y} : {\bf y} \in \square \}\subset \R^d$ 
of bounded domains, in space dimension $d=2,3$.
Here, we
are looking for a tuple 
$({\bf u}_{\bf y},p_{\bf y})\in \calX_{{\bf y}}
:= 
\big[H_0^1\big(D_{\bf y} \big)\big]^d \times L_{\#}^2\big(D_{\bf y} \big)$ 
such that
\begin{equation}\label{eq:ns}
%====================================
{\mathrm L}\left( \begin{array}{c} {\bf u}_{\bf y} \\ p_{\bf y} \end{array} \right) 
:= 
\left\{\;\begin{aligned}
-\Delta {\bf u}_{\bf y} + \big({\bf u}_{\bf y} \cdot\nabla\big)
	{\bf u}_{\bf y} + \nabla p_{\bf y} &= {\bf f}&&\text{in}\ D_{\bf y}\,,\\[1ex]
	\div {\bf u}_{\bf y} &= 0&&\text{in}\ D_{\bf y}\,,\\[1ex]
{\bf u}_{\bf y} &= \mathbf{0}&&\text{on}\ \Gamma_{\bf y}\,.
\end{aligned}\right.
\end{equation}
Here,
the space $L_{\#}^2\big(D_{\bf y} \big)$ denotes the space of
all square-integrable functions with vanishing mean, 
that is
\[
L_{\#}^2\big(D_{\bf y} \big) \isdef \Bigg\{ p\in L^2\big(D_{\bf y} \big): 
  	\int_{D_{\bf y}} p \d{\bf x} = 0\Bigg\}\,.
\]
Classical arguments based on the Faedo-Galerkin method, 
Brouwer’s fixed point theorem and compactness, 
ensure the existence of variational solutions in the sense of E. Hopf and J. Leray 
(see, e.g., \cite[Chpt.~II, Sct.~1, Thm.~1.2]{Temam}).
Under additional so-called ``small-data'' assumptions, these solutions are unique.

Shape holomorphy for (Leray-Hopf solutions of)
the Navier-Stokes equations \eqref{eq:ns} 
has been verified firstly in \cite{CSZ}.
There, the Leray-Hopf solutions have been extended to complex data, 
and have been shown to be continuously Fr\'{e}chet-differentiable 
with respect to the domain, which implies holomorphic dependence of
solutions on the shape by a classical result in complex analysis in 
Banach spaces.
Real analyticity has subsequently been proven in \cite{Chernov}, 
via bootstrapping of the parametric regularity and using the chain rule.
Specifically, there holds the following
result on growth-bounds of the derivatives of the solution ${\bf u}_{\bf y}$
and $p_{\bf y}$ of the Navier-Stokes equations \eqref{eq:ns}
with respect to the parameter ${\bf y}$.
Note that this result, of course, applies in particular 
to the stationary Stokes equations.
\begin{theorem}[{\cite[Thm.~2]{Chernov}, \cite{CSZ}}]\label{thm:ns}
%==============================================
Let ${\bf u}_{\bf y} \in \big[H_0^1\big(D_{\bf y}\big)\big]^d$ 
and $p_{\bf y} \in L_{\#}^2\big(D_{\bf y} \big)$ 
be the solution 
to the Navier-Stokes equations \eqref{eq:ns} and 
$\hat{\bf u}_{\bf y} \isdef {\bf u}_{\bf y}\circ{\bf V}_{\bf y}^{-1} \in [H_0^1(D_\refd)]^d$ 
and 
$\hat{p}_{\bf y} \isdef p_{\bf y}\circ{\bf V}_{\bf y}^{-1} \in L_{\#}^2(D_\refd)$ 
the respective pull-backs onto the reference domain $D_{\refd}$. 
If the loading ${\bf f}\in [C^\infty(\mathcal{D})]^d$ is analytical 
and the deformation field ${\bf V}_{\bf y}$ 
satisfies Assumption~\ref{ass:V}, then there exists a
$\rho > 0$ such that we have
\[
\forall {\balpha}\in\calF: \quad 
  \|\partial^{\balpha}_{{\bf y}}\hat{\bf u}\|_{[H^1(D_\refd)]^d} 
  	\lesssim |\balpha|! \rho^{|\balpha|}\bgamma^{\balpha}\,,\quad
  \|\partial^{\balpha}_{{\bf y}}\hat{p}\|_{L^2(D_\refd)} 
  	\lesssim |\balpha|! \rho^{|\balpha|}\bgamma^{\balpha}
\,.
\]
\end{theorem}
%==========================================================
\subsubsection{Boundary integral equations}
%==========================================================
If an acoustic wave encounters an impenetrable, bounded 
obstacle $D_{\bf y}\subset\mathbb{R}^3$, having a Lipschitz 
smooth boundary $\Gamma_{\bf y}\isdef\partial D_{\bf y}$, then 
it gets scattered. Given the ``incident plane wave'' $u^{\text{inc}}
({\bf x}) = e^{i\kappa\langle{\bf d},{\bf x}\rangle}$ with known wavenumber
$\kappa$ and direction ${\bf d}$, where $\|{\bf d}\|_2=1$, then the 
superposition $u_{\bf y} = u^{\text{inc}}+u^{\text{s}}_{\bf y}$ of the 
incident plane wave and the ``scattered wave'' satisfies 
the exterior Helmholtz equation
\begin{equation}\label{eq:pde1}
%================================
  \Delta u_{\bf y} + \kappa^2 u_{\bf y} = 0\ \text{in}\ \mathbb{R}^3\setminus\overline{D_{\bf y}}\,.
\end{equation}
The boundary condition at the scatterer's surface
depends on its physical properties. If the scatterer 
constitutes a ``sound-soft obstacle'', then 
the acoustic pressure vanishes at \(\Gamma\)
and we have the homogeneous Dirichlet condition 
\begin{equation}\label{eq:pde2a}
%================================
  u_{\bf y} = 0\ \text{on}\ \Gamma_{\bf y}\,.
\end{equation}
Whereas, if the scatterer constitutes a so-called 
``sound-hard obstacle'', 
then the normal velocity vanishes at \(\Gamma_{\bf y} \)
and we have the homogeneous Neumann condition 
\begin{equation}\label{eq:pde2b}
%================================
\frac{\partial u_{\bf y}}{\partial \bf n} = 0\ \text{on}\ \Gamma_{\bf y}\,.
\end{equation}
The behavior at infinity is imposed 
by the Sommerfeld radiation condition
\begin{equation}	\label{eq:pde3}
%======================================
  \lim_{r\to\infty} r \left\{ \frac{\partial u^{\text{s}}_{\bf y}}{\partial r}({\bf x})
  	- i\kappa u_{\bf y}^{\text{s}}({\bf x})\right\} = 0\,,\ \text{where}\ r\isdef\|{\bf x}\|_2\,.
\end{equation}
In order to solve the boundary value problem
\eqref{eq:pde1}--\eqref{eq:pde3}, we shall employ 
a suitable reformulation by boundary integral equations.
To this end, we introduce the acoustic single and double 
layer operators, the adjoint of the double layer operator, 
and the hypersingular operator given by
\begin{align*}
  \mathcal{V}_{\bf y} \colon H^{-1/2}\big(\Gamma_{\bf y} \big)\to H^{1/2}\big(\Gamma_{\bf y} \big)\,,
  &\quad(\mathcal{V}_{\bf y}\rho)({\bf x}) 
    := \int_{\Gamma_{\bf y}}
  	G({\bf x},{\bf x}')\rho({\bf x}')\d\sigma_{{\bf x}'}\,,
  \\
  \mathcal{K}_{\bf y} \colon L^2\big(\Gamma_{\bf y} \big)\to L^2\big(\Gamma_{\bf y} \big)\,,
  &\quad(\mathcal{K}_{\bf y}\rho)({\bf x}) := 
        \int_{\Gamma_{\bf y} } \frac{\partial G({\bf x},{\bf x}')} {\partial{\bf n}_{{\bf x}'}}\rho({\bf x}')\d\sigma_{{\bf x}'}\,,
  \\
  \mathcal{K}_{\bf y}^\star \colon L^2\big(\Gamma_{\bf y} \big)\to L^2\big(\Gamma_{\bf y} \big)\,,
  &\quad(\mathcal{K}_{\bf y}^\star\rho)({\bf x}):=\int_{\Gamma_{\bf y} }
  	\frac{\partial G({\bf x},{\bf x}')}
	{\partial{\bf n}_{\bf x}}\rho({\bf x}')\d\sigma_{{\bf x}'}\,,\\
  \mathcal{W}_{\bf y}\colon H^{1/2}\big(\Gamma_{\bf y} \big)\to H^{-1/2}\big(\Gamma_{\bf y} \big)\,,
  &\quad(\mathcal{W}_{\bf y}\rho)({\bf x}) := -\frac{1}{\partial{\bf n}_{\bf x}}\int_{\Gamma_{\bf y} }
  	\frac{\partial G({\bf x},{\bf x}')}{\partial{\bf n}_{{\bf x}'}}
		\rho({\bf x}')\d\sigma_{{\bf x}'}\,.
\end{align*}
Here, ${\bf n}_{\bf x}$ and ${\bf n}_{{\bf x}'}$ denote the 
outward pointing normal vectors at the surface points 
${\bf x}, {\bf x}'\in\Gamma_{\bf y} $, respectively, while 
$G(\cdot,\cdot)$ denotes the fundamental solution for the Helmholtz equation
given by
\[
  G({\bf x}, {\bf x}') = \frac{e^{i\kappa\|{\bf x}-{\bf x}'\|_2}}{4\pi\|{\bf x}-{\bf x}'\|_2}\,.
\]

Although the Helmholtz problem \eqref{eq:pde1}--\eqref{eq:pde3}
is uniquely solvable, the respective boundary integral formulation 
might not if $\kappa^2$ is an eigenvalue for the Laplacian inside 
the scatterer $D$. In order to avoid such \emph{spurious modes}, 
one can employ combined field integral equations. 
Then, for some real $\eta\neq 0$, the solution of the boundary integral equation
\begin{equation}\label{eq:D2N}
%===================================
  \bigg(\frac{1}{2}+\mathcal{K}_{\bf y}^\star - i\eta\mathcal{V}_{\bf y} \bigg)
  	\frac{\partial u_{\bf y}}{\partial{\bf n}} 
         = \frac{\partial u^{\text{inc}}}{\partial{\bf n}}
		-i \eta u^{\text{inc}}\quad\text{on}\ \Gamma_{\bf y}
\end{equation}
yields the unknown Neumann data in case of sound-soft scattering 
problems. In case of sound-hard obstacles, the boundary integral equation
\begin{equation}\label{eq:N2D}
%===================================
  \bigg(\frac{1}{2} - \mathcal{K}_{\bf y}  
          + i\eta \mathcal{W}_{\bf y} \bigg)u_{\bf y}
  = u^{\text{inc}}-i\eta\frac{\partial u^{\text{inc}}}{\partial{\bf n}}
  	\quad\text{on}\ \Gamma_{\bf y}
\end{equation}
yields the unknown Dirichlet data. Notice that the boundary 
integral equations \eqref{eq:D2N} and \eqref{eq:N2D} are 
always uniquely solvable, independent of the wavenumber 
$\kappa$, see \cite{BM,CK2,Kuss1969}. 
Note that again the 
total field $\hat{u}_{\bf y} = u\circ{\bf V}_{\bf y}^{-1}$ and 
hence $\hat{u}^{\text{s}}_{\bf y} = u^{\text{s}}\circ{\bf V}_{\bf y}^{-1}$ 
admit the desired grow bounds \eqref{eq:regu} with
respect to the $\|\cdot\|_{H^{1/2}(\Gamma_\refd)}$-norm and
likewise the associated Neumann data with respect to the 
$\|\cdot\|_{H^{-1/2}(\Gamma_\refd)}$-norm. Especially, 
also the boundary integral operators and associated layer 
potentials are analytic with respect to the parameter ${\bf y}$, 
see \cite{JDFH24,HenChS21} for example.
Similar results are available for BIEs corresponding to 
the time-harmonic Maxwell equations, see \cite{JSZ17_2339}.
In accordance with Subsection~\ref{sec:interfaces}, 
we finally conclude that skeleton integral equation formulations 
as proposed in \cite{GHS25} are analytic in the 
skeleton-parametrization subject to parameter ${\bf y}\in\square$.
%==============================================
\subsection{Parabolic initial boundary value problems}
\label{sec:ParIBVP}
%==============================================
The PDEs considered up to this point are stationary and elliptic.
Analyticity of the domain-to-solution map is, however, also afforded
for certain time-dependent PDEs, which have recently been investigated
(see, e.g., \cite{Bruegger1,Bruegger2,para,ana}).
We illustrate this for a basic linear, 
parabolic initial-boundary value problem (IBVP for short). 
\subsubsection{Space-time operator equation}
\label{sec:xtOpEq}
In a finite time interval $[0,T]$ and a 
bounded spatial domain $D \subset \mathbb{R}^d$,
consider the linear, parabolic PDE
\begin{equation}\label{eq:ParIBVP}
Bu = \big(\partial_t + \mathrm{L}({\bf A} , \partial_{{\bf x}})\big) u 
= f \quad\mbox{in}\quad [0,T]\times D\,,
\end{equation} 
where, as in Section~\ref{sec:Poisson} and Secion~\ref{sec:interfaces}, 
the spatial elliptic operator is $\mathrm{L}({\bf A} , \partial_{{\bf x}}) := 
-\div ({\bf A} \nabla)$ with the diffusion coefficient ${\bf A}$ in the set 
of admissible data $\calD(a_-,a_+)$ as defined in \eqref{eq:Da+a-}.

The PDE \eqref{eq:ParIBVP} is complemented 
with initial- and boundary conditions,
\begin{equation}\label{eq:IniBdry}
u({\bf x},0) = u_0({\bf x}) \ \text{in}\ D\,,\quad 
\text{and}\quad u|_{[0,T]\times \partial D} = 0\,,
\end{equation}
respectively.
We impose homogeneous Dirichlet boundary conditions 
on all of $[0,T]\times \partial D$ for convenience -- all that 
follows shall remain valid also for mixed, or Neumann 
boundary conditions on $[0,T]\times \partial D$. 

We consider the usual, 
abstract setting for IBVP \eqref{eq:ParIBVP}\&\eqref{eq:IniBdry},
see e.g.~\cite{DL92,Wloka}, which is based on a Gel'fand triplet 
$V\subset H \simeq H' \subset V'$ 
where
\[
V = H^1_0(D)\,,\quad H = L^2(D) \simeq H'\,,\quad V' = H^{-1}(D) \,.
\]
Problem \eqref{eq:ParIBVP} is a special instance of an 
\emph{abstract parabolic setting}: 
given sesquilinear forms 
$\{ a(t;\cdot,\cdot): V\times V\to \mathbb{R},\ t\in [0,T] \}$
which are continuous and coercive uniformly with respect to $t\in [0,T]$, i.e.\
there are constants  $M_a$, $\alpha > 0$, and $\lambda \geq 0$
such that for $0\leq t \leq T$
\begin{align*}
\forall w,v \in V:&\quad |a(t;w,v)| \leq M_a \| w \|_V \| v \|_V\,,\\
\forall v\in V:&\quad \operatorname{Re} a(t;v,v) + \lambda \| v \|_H^2 \geq \alpha \| v \|_V^2
\,.
\end{align*}
Then, for $t\in [0,T]$, 
there is a unique operator $A(t)\in \calL(V,V')$ 
representing the form $a(t;\cdot,\cdot)$. Hence, given $u_0\in H$ 
and $\{f(t): 0<t<T\} \subset V'$, we consider the abstract 
parabolic problem
\begin{equation}\label{eq:IBVPOp}
\frac{\d}{\d t} u(t) + A(t)u(t) = f(t)\ \mbox{in}\ V'\,, \quad u(0) = u_0\ \mbox{in}\ H\,.
\end{equation}
This can be interpreted as an initial-value ODE in function space, which 
corresponds to what is done in semigroup theory for example (e.g.\ \cite{Pazy}).

We opt instead for 
the \emph{space-time operator formulation} of IBVP \eqref{eq:IBVPOp} 
which is based on the Bochner spaces 
\begin{equation}\label{eq:Bochner}
\calX = L^2(I;V) \cap H^1(I;V')\,, \quad 
\calY = L^2(I;V) \times H\,.
\end{equation}
The intersection space $\calX$ is endowed with the sum-type norm
$$ 
\| v \|_\calX := \left( \| v \|_{L^2(I;V)}^2 + \Big\| \frac{\d v}{\d t}\Big\|_{L^2(I;V')}^2 \right)^{1/2},
$$
and, for $v = (v_1,v_2) \in \calY$, the norm on $\calY$ is defined as
$$
\| v \|_{\calY} :=  \left( \| v_1 \|_{L^2(I;V)}^2 + \| v_2 \|_H^2 \right)^{1/2}.
$$
With these spaces in place, 
the weak form of the space-time operator $B$ in \eqref{eq:ParIBVP} is 
$B \in \calL(\calX,\calY')$ which corresponds, 
for given diffusion coefficient ${\bf A}\in \calD(a_-,a_+)$  
to the bilinear form $b({\bf A}; \cdot,\cdot):\calX\times \calY\to \mathbb{R}$ 
given by
$$
b\big(w;(v_1,v_2)\big) 
:= 
\int_I \Big\langle \frac{\d w}{\d t}(t), v(t) \Big\rangle_H + a\big(t;w(t),v(t)\big) \d t
+
\langle w(0), v_2 \rangle_H 
$$
and the data $f \in \calY'$ and $u_0 \in H$ are combined into the 
``load functional'' $\ell\in \calY'$ defined by 
$$
\ell(v) := \int_I \langle f(t), v_1(t)\rangle_H \d t +  \langle u_0,v_2 \rangle_H\,, 
\quad v = (v_1,v_2) \in \calY\,.
$$
The space-time operator formulation of the IBVP \eqref{eq:ParIBVP} then reads: 
given an admissible diffusion coefficient ${\bf A} \in \calD(a_-,a_+)$ 
and a pair $(f,u_0) \in \calY'$, 
find $u_{\bf A} \in \calX$ such that
\begin{equation}\label{eq:IBVPweak}
b({\bf A}; u, v) = \ell(v) \quad \forall v\in \calY\,.
\end{equation}
This equation is uniquely solvable in virtue of 
$b(\cdot; \cdot , \cdot)$ satisfying the 
uniform inf-sup condition: 
there is a constant $\gamma>0$ (depending on $a_\pm$, $T$, and $D$) 
such that for every diffusion coefficient ${\bf A} \in \calD(a_-,a_+)$
there holds (see, e.g., \cite[Appendix]{ScSt09} for a proof; it is 
straightforward from inspection of this argument that the constant
$\gamma>0$ is uniform with respect to ${\bf A} \in \calD(a_-,a_+)$)
\begin{equation}\label{eq:Parinfsup}
\inf_{0\ne w\in \calX} \sup_{0\ne v\in \calY} \frac{b({\bf A}; w,v)}{\|w \|_\calX \|v\|_\calY} \geq  \gamma
\;\; \mbox{and} \; \;
\inf_{0\ne v\in \calY} \sup_{0\ne w\in \calX} \frac{b({\bf A}; w,v)}{\|w \|_\calX \|v\|_\calY} \geq  \gamma\,.
\end{equation}
Equivalently, for every diffusion coefficient ${\bf A} \in \calD(a_-,a_+)$,
the underlying operator
$$
B_{\bf A} \in \calL_{\iso}(\calX,\calY') 
$$
is boundedly invertible.
Whence, 
for every diffusion matrix ${\bf A} \in \calD(a_-,a_+)$, and 
for every data $(f,u_0) \in \calY'$, 
there exists a unique solution $u_{\bf A} \in \calX$ of the (well-posed) 
IBVP in the space-time operator equation form \eqref{eq:IBVPweak}.
%%%%%%%%%%%%%%%%%%%%%%%%%%%%%%%%%%%%%%%%%%%%%%%%%%%%%%%%%%%%%%%%%%%%%%%%%
\subsubsection{Analyticity of the data-to-solution map}
\label{sec:AnDat2Sol}
%%%%%%%%%%%%%%%%%%%%%%%%%%%%%%%%%%%%%%%%%%%%%%%%%%%%%%%%%%%%%%%%%%%%%%%%%
\begin{theorem} \label{thm:ParAnl}
The data-to-solution map 
\begin{equation}\label{eq:ParabSol}
{\mathrm S}: \calD(a_-,a_+) \times \calY' \to \calX: ({\bf A}, f, u_0) \mapsto u 
\end{equation}
is analytic.
\end{theorem}

\begin{proof}
We sketch proof, which proceeds along the lines of the argument in 
the proof for the elliptic case, Theorem~\ref{thm:regu}.

The map ${\mathrm S}$ is linear with respect to the arguments $f$ and $u_0$, 
hence analytic in these components. 
We fix therefore $f$ and $u_0$ and consider only the section  
$$
{\mathrm S}|_{\calD} : \calD(a_-,a_+) \to \calX: {\bf A} \mapsto u\,.
$$
We observe that \eqref{eq:Parinfsup} implies that 
\begin{enumerate} 
\item 
the data-to-operator map ${\bf A} \mapsto B_{\bf A}$ is linear, 
hence analytic,
\item 
for all ${\bf A}\in \calD(a_-,a_+)$, $B_{\bf A}$ is boundedly invertible
due to the uniformity with respect to ${\bf A}\in \calD(a_-,a_+)$ 
of the inf-sup conditions \eqref{eq:Parinfsup},
\item 
the inversion map 
${\rm inv}: \calL_{\iso}(\calX,\calY') \to \calL_{\iso}(\calY',\calX) :  B_{\bf A} \mapsto B_{\bf A}^{-1}$ 
is analytic,
\item 
for $C\in \calL(\calY',\calX)$ the evaluation map 
$$
{\rm apply}: \; \calL(\calY',\calX)\times \calY' \to \calX: (f,u_0) \mapsto u = C(f,u_0)
$$
is linear, hence analytic.
\end{enumerate}
We conclude that for fixed $u_0$ and $f$, 
on $\calD(a_-,a_+)$ 
the data-to-solution map 
$$
{\mathrm S}|_{\calD}: {\bf A} \mapsto u = {\rm apply}\big({\rm inv} (B_{\bf A}),(f,u_0)\big) 
$$
is a composition of analytic maps, hence real analytic.
This and the linearity of $\calS$ with respect to $f$ and $u_0$ 
imply analyticity of \eqref{eq:ParabSol}.
\end{proof}
%
%%%%%%%%%%%%%%%%%%%%%%%%%%%%%%%%%%%%%%%%%%%%%%%%%%%%%%%%%%%%%%%%%%%%%%%%%
\subsubsection{Analyticity of the domain-to-solution map}
\label{sec:AnSh2Sol}
%%%%%%%%%%%%%%%%%%%%%%%%%%%%%%%%%%%%%%%%%%%%%%%%%%%%%%%%%%%%%%%%%%%%%%%%%
Theorem~\ref{thm:ParAnl} implies 
the shape-analyticity of the parabolic IBVP \eqref{eq:ParIBVP}
for \emph{time-independent domain transformations} 
as introduced in Section~\ref{sec:PrbFrm}. 
For $D_{\refd}$ as in \eqref{eq:ParDom} with ${\bf V}$ as in Assumption~\ref{ass:V}, 
we denote by $\calX({\bf V})$ and $\calY({\bf V})$ the Bochner spaces 
with respect to the spatial domain ${\bf V}(D)$ in \eqref{eq:ParDom}.

\begin{theorem}\label{thm:AnSh2Sol}
Consider the 
parabolic initial-boundary value problem \eqref{eq:ParIBVP}, \eqref{eq:IniBdry}
on the parametric family $\{ [0,T] \times D_{{\bf y}} : {\bf y}\in \Box \}$ 
of domains, with the parametric domain $D_{{\bf y}}$ as 
defined in \eqref{eq:ParDom}.
Suppose that Assumption~\ref{ass:V} and 
the uniformity condition \eqref{eq:uniformity} hold.

Then, 
for each ${\bf y} \in \Box$ the IBVP \eqref{eq:ParIBVP}, \eqref{eq:IniBdry}
in $[0,T] \times D_{{\bf y}}$ 
admits a unqiue solution $u_{{\bf y}} \in \calX_{{\bf y}}$.
The dependence of $u({\bf V})$ on the vector field ${\bf V}_{{\bf y}}$ 
as in Assumption~\ref{ass:V} is analytic.

For the affine-parametric family \eqref{eq:KL}, 
the corresponding parametric family
$\{ \hat{u}_{{\bf y}} = u \circ {{\bf V}_{{\bf y}}^{-1}} \in \calX_{\refd} : {\bf y}\in \square\}$
of solution pullbacks to $[0,T] \times D_{\refd}$ 
depends analytically on $y_j \in {\bf y}$.
I.e., 
the pullback $\hat{u}_{{\bf y}}$ to $[0,T] \times D_{\refd}$ 
satisfies, 
with $\calX_{\refd}$ in place of $\calX$ as in \eqref{eq:Bochner} on $[0,T]\times D_{\refd}$
and 
with the sequence $\bgamma \in (0,\infty)^{\mathbb{N}}$ in \eqref{eq:gammak},
that
there is a constant $\rho>0$ such that
\[
\forall {\bf y} \in \square\; \forall \balpha \in \calF: \quad
\| \partial^\balpha_{\bf y} \hat{u}_{{\bf y}} \|_{\calX_{\refd}} 
\lesssim 
|\balpha|!\rho^{|\balpha|}\bgamma^\balpha\,.
\]
\end{theorem}
\begin{proof}
Observe that the pull back to $D_{\refd}$ of the parabolic operator $B_{\refd}$ 
has the form \eqref{eq:ParIBVP} with parameter sequence 
${\bf y}\in \square$ and with diffusion coefficient ${\bf A}$ given by
$$
{\bf A}_{\bf y} \leftarrow 
{\bf J}_{\bf y}({\bf x})^{-1}  ({\bf A}\circ {\bf V}_{\bf y})({\bf x}) {\bf J}_{\bf y}({\bf x})^{-\intercal}
\det{\bf J}_{\bf y}({\bf x})\,,
\quad 
{\bf x}\in D_{\refd}\,.
$$
In particular,
${\bf A}_{\bf y}$ belongs to the data class
$\calD(a_{\refd,-} , a_{\refd,+})$,
where the constants 
$0< a_{\refd,-} \leq a_- \leq a_+ \leq a_{\refd,+} <\infty$ 
depend on $a_{\pm}$ in \eqref{eq:Da+a-} 
and on $\overline{\sigma}, \underline{\sigma}$ in \eqref{eq:Jbound}.
The conclusion now follows from Theorem~\ref{thm:ParAnl}.
\end{proof}
Remark~\ref{rmk:Gs} also applies for the parabolic PDE \eqref{eq:ParIBVP}\&\eqref{eq:IniBdry}:
If the domain transformations are Gevrey-$s$ regular, the parameter-to-solution map 
will likewise be Gevrey-$s$ regular as a consequence of Theorem~\ref{thm:ParAnl} 
and the Gevrey-composition theorem \cite[Thm.\ 2.3]{HSS}. A parametric 
regularity result of this sort was recently obtained in \cite{ana}.

%==========================================================
\section{Neural and spectral operator surrogates}
\label{sec:NeurSpcOpSr}
%==========================================================
Having established holomorphy of data-to-solution maps 
for several PDE and BIE models, 
we show that these maps can be `learned' by finite-parametric, 
numerically accessible surrogates $\calG_N$, 
to any prescribed accuracy. 
We establish in addition expression rates for 
spectral and neural surrogates.
To this end, we recap recent results from \cite{HSZ} 
in a form related to the present setting,
particularly to the affine-parametric PCA shape encodings \eqref{eq:KL} 
of the vector field ${\bf V}_{{\bf y}}$ occurring in the domain maps.
%==========================================================
\subsection{Shape encodings}
\label{sec:ParEnc}
According to Assumption~\ref{ass:V}, 
we consider parametric families of shapes
which are encoded via \eqref{eq:ParDom}--\eqref{eq:gammak}. 
Pullback of the PDE and the associate solution
to the reference domain via the affine-parametric 
vector field ${\bf V}_{{\bf y}}$ in \eqref{eq:KL}
provides an \emph{equivalent, parametric PDE on the 
nominal domain $D_{\refd}$},
with input data in a (subset of a separable, Hilbertian) data space $\calZ$,
parametrized by sequences ${\bf y} = (y_j)_{j\in\N} \in \square$.
See, e.g., \eqref{eq:para2data} and \eqref{eq:ParabSol}.

As explained in Section~\ref{sec:Expls}, for 
a wide range of operator equations,
standard (elliptic resp.~parabolic) PDE regularity 
will imply solution regularity corresponding to suitable data regularity.
To quantify it, we introduce scales 
$\{\calZ^s\}_{s\geq 0}$ , $\{ \calX^t \}_{t\geq 0}$ of subspaces, 
with the convention that $\calZ = \calZ^0$, $\calX = \calX^0$,
and the monotonicity $\calX^{t'} \subset \calX^{t}$ for $t' > t \geq 0$ etc.
%%%%%%%%%%%%%%%%%%%%%%%%%%%%%%%%%%%%%%%%%%%%%%%%%%%%%%%%%% 
\subsection{Representation systems}
\label{sec:BiOrth}
%%%%%%%%%%%%%%%%%%%%%%%%%%%%%%%%%%%%%%%%%%%%%%%%%%%%%%%%%%
Representation systems are bases or, more generally, 
countable frames in function spaces. 
We shall use these in the 
constructions of linear en- and decoders and corresponding pairs
of encoder/decoder maps. 
Specifically, 
we assume that shape is encoded via the parametric
vector field ${\bf V}_{\bf y}$ as specified in Assumption~\ref{ass:V}.
For the decoder construction, 
we will admit (possibly redundant) biorthogonal systems 
such as Riesz bases realized by
Multiresolution Analyses (MRAs) or by 
Finite Element or B-spline frames (e.g.\ \cite{HSS08}).
These can be realized numerically on 
hierarchies of unstructured, regular triangulations
via multi-level algorithms such as the so-called BPX preconditioner.
Other examples are Fourier bases, 
as used e.g.~in the Fourier Neural Operators, which are orthonormal
systems. 
See e.g.~\cite[Chpt.\ 2,5,6.3]{TriebelWavelets2008} 
for a general discussion and constructions of 
stable MRAs in function spaces on domains and manifolds.
%%%%%%%%%%%%%%%%%%%%%%%%%%%%%%%%%%%%%%%%%%%%%%%%%%%%%
  \subsubsection{Frames}
  \label{sec:Frames}
%%%%%%%%%%%%%%%%%%%%%%%%%%%%%%%%%%%%%%%%%%%%%%%%%%%%%
  Constructions of representation systems are often
  simplified when one insists on stability, 
  but relaxes the basis property, i.e.\ allows some redundancy.
  This leads to the concept of \emph{frames} \cite{ChrstnsenFramesBases2008,Heil}.
  It comprises biorthogonal wavelet bases as a
  particular case, and allows in particular also iterative realization
  on unstructured simplicial partitions of polyhedra via the so-called
  BPX multi-level iteration (see, e.g., \cite{HSS08}, and the orginal
  construction \cite{OswaldBPXP12d} due to P.\ Oswald in space dimension $d=2$,
  and subsequently in polyhedra in \cite{OswaldML94} and the references
  there. A frame property in $H^1_0(\domain)$ 
  of the subspace splittings furnished by the BPX iteration is also implicitly
  shown in \cite{HSS08}).

  \begin{definition} \label{def:Frame}
  Let $\calZ$ denote a real, separable Hilbert space.
  A collection $\bsPsi = (\psi_j)_{j\in\N}\subset \calZ$ 
  is called a \emph{frame for $\calZ$},
  if the \emph{analysis operator}
$$
\rF: \calZ \to \ell^2(\N)\,,\quad v\mapsto (\dup{v}{\psi_j})_{j\in\N}
$$
is boundedly invertible between $\calZ$ and ${\rm range}(\rF)\subseteq \ell^2(\N)$.
\end{definition}
The adjoint $\rF'$ of the analysis operator is 
the \emph{synthesis operator}, given by
\begin{equation}\label{eq:calFp}
\rF': \ell^2(\N) \to \calZ\,,\quad \bsv \mapsto \bsv^\intercal\bsPsi \,.
\end{equation}
\emph{Numerical stability of frames} is quantified by the \emph{frame bounds}
\be \label{eq:FrBd} 
\lambda_{\bsPsi} := \inf_{0\ne v \in \calZ} \frac{\| \rF v \|_{\ell^2}}{\| v \|_\calZ}\,,
\quad
\Lambda_{\bsPsi} := \| \rF \|_{\calZ \to \ell^2} 
                  = \sup_{0\ne v \in \calZ }  \frac{\| \rF v \|_{\ell^2}}{\| v \|_\calZ} \,.
\ee
\emph{Parseval frames} are frames $\bsPsi$ with ideal conditioning 
$\lambda_{\bsPsi} = \Lambda_{\bsPsi} = 1$.
\begin{remark}\label{rmk:FrBd}
    Since 
    $\|\rF'\|_{\ell^2\to\calZ } = \|\rF\|_{\calZ \to\ell^2}$, 
    \eqref{eq:FrBd} implies that for all $\bsv\in\ell^2(\N)$
    \[
      \normlr[\calZ]{\sum_{j\in\N}v_j\psi_j}^2
      = \norm[\calZ]{\rF'\bsv}^2
      \le \Lambda_{\bsPsi}^2\sum_{j\in\N}v_j^2 
      = \Lambda_{\bsPsi}^2 \norm[\ell_2]{\bsv}^2\,.
    \]
\end{remark}

The \emph{frame operator} 
$\rS := \rF'\rF: \calZ \to \calZ$ is 
boundedly invertible, self-adjoint and positive
(e.g.~\cite{ChrstnsenFramesBases2008,Heil}) 
with
$\| \rF' \rF \|_{\calZ \to \calZ} = \Lambda_{\bsPsi}^2$ 
and
$\| (\rF' \rF)^{-1} \|_{\calZ\to \calZ} = \lambda_{\bsPsi}^{-2}$,
\cite[Lemma 5.1.5]{Christensen}.
With the pseudoinverse $(\rF')^\dagger = \rF(\rF'\rF)^{-1}$ 
of the synthesis operator, 
$$
(\rF')^\dagger: \calZ \to \ell^2(\N)\,,
\quad f \mapsto \{ \langle f, \rS^{-1} \psi_j \rangle \}_{j\in \N}\,,
$$
the \emph{frame decomposition theorem} asserts that
every $f\in \calZ$ can be uniquely and stably reconstructed
from a corresponding sequence of frame coefficients via
$$
f 
= \rF'(\rF')^\dagger f 
= \sum_{j\in \N} \langle f,\rS^{-1}\psi_j \rangle \psi_j
= \sum_{j\in \N} \langle f,\psi_j\rangle \rS^{-1} \psi_j 
\,.
$$
This highlights the importance of
the collection
$\widetilde{\bsPsi}:= \rS^{-1}\bsPsi$. 
It
is a frame for $\calX$ which is referred to 
as the \emph{canonical dual frame} of $\bsPsi$.  
Its analysis operator is
$\widetilde{\rF} := (\rF')^\dagger = \rF (\rF'\rF)^{-1}$,
and its frame bounds
\eqref{eq:FrBd} are $\lambda_{\bsPsi}^{-1}$ and
$\Lambda_{\bsPsi}^{-1}$, respectively.

The frame decomposition theorem takes the form 
(e.g.~\cite{ChrstnsenFramesBases2008,Heil})
\[
  \rF' \widetilde{\rF} = I\quad\text{on } \calX\,.
\]
Whence every $v\in \calZ$ has a representation 
$v = \bsv^\intercal\bsPsi$ with $\bsv = \widetilde\rF(v)\in \ell^2(\N)$, 
and
\be\label{eq:FrStab} 
\Lambda_\bsPsi^{-1} \leq \frac{\| \bsv \|_{\ell^2}}{\|v\|_\calZ} \leq \lambda_{\bsPsi}^{-1} \,.  
\ee
Property \eqref{eq:FrStab} is equivalent to $\bsPsi$ 
being a frame for $\calZ$ (see, e.g., \cite[Thm.~8.29 (b)]{Heil}).  

\begin{remark}[Continuous-discrete equivalence]\label{rmk:CDE} 
Given a (generally nonlinear) map $\calG:\calZ \to \calX$ 
between the separable Hilbert space $\calZ$ encoding the shape and 
the solution space $\calX$ over $\R$
which are each endowed with frames $\bsPsi_\calZ$ and $\bsPsi_\calX$, 
respectively,
there exists a \emph{coordinate map} 
$\bG = \bG(\bsPsi_\calZ, \bsPsi_\calX):\ell_2(\N) \to \ell_2(\N)$ 
such that there holds 
\begin{equation}\label{eq:CDE}
\calG 
= \rF'_\calZ \circ \bG \circ (\rF'_\calX)^\dagger 
= \rF'_\calZ \circ \bG \circ \widetilde{\rF}_\calX
\,.
\end{equation}
Here and throughout, 
in case that the representation systems 
$\bsPsi_\calZ$ and $\bsPsi_\calX$ 
are clear from the context, 
we write $\bG$ in place of $\bG(\bsPsi_\calZ, \bsPsi_\calX)$.
\end{remark}
\begin{remark}[Uniqueness of coordinate sequence]
 \label{rmk:FramKer}
Unless
$\bsPsi$ is a Riesz basis (see below), 
the representation of $v\in \calX$ as
$v = \bsv^\intercal\bsPsi$ is generally not unique: 
there holds
$\ell^2(\N) = {\rm ran}(\rF) \oplus^\perp {\rm ker}(\rF')$ 
and
$\bsQ := \widetilde{\rF} \rF'$ is the orthoprojector onto
${\rm ran}(\rF)$ \cite{ChrstnsenFramesBases2008,Heil}.
\end{remark}
The frame-based shape-encoding and solution decoding 
covers a wide range of numerical approximations.
We use in Assumption~\ref{ass:V} 
the \KL-based shape encoding \eqref{eq:KL}, 
which is an orthonormal basis.
Admitting non-orthogonal affine-parametric representation 
such as \eqref{eq:KL} also covers 
certain B-splines 
with frame properties \cite[Chpt.~6]{ChrstnsenFramesBases2008}.

%%%%%%%%%%%%%%%%%%%%%%%%%%%%%%%%%%%%%%%%%%%%%%%%%%%%%%%%%%%%%%%%%%%%%%%%%%%%%%%%%%%%%%%%%%%%%%
\subsubsection{Riesz bases}
\label{sec:RieszB}
%%%%%%%%%%%%%%%%%%%%%%%%%%%%%%%%%%%%%%%%%%%%%%%%%%%%%%%%%%%%%%%%%%%%%%%%%%%%%%%%%%%%%%%%%%%%%%
%
\begin{definition}[\cite{ChrstnsenFramesBases2008,Heil}]\label{def:RieszB}
  In a separable Hilbert space $\calX$, 
  a sequence $\bsPsi = (\psi_j)_{j\in \N}\subset \calX$ is a
  Riesz basis of $\calX$ if there exists a bounded bijective operator
  $A:\calX\to\calX$ and an $\calX$-orthonormal basis $(e_j)_{j\in\N}$ 
  such that $\psi_j=Ae_j$ for all $j\in\N$.
\end{definition}
Every Riesz basis of $\calX$ is a frame without redundancy:
the frame bounds \eqref{eq:FrBd} coincide with 
\emph{Riesz constants}
$0< \lambda_\bsPsi \leq \Lambda_\bsPsi <\infty$ 
such that
\[
\forall (c_j)_{j\in\N}\in\ell^2(\N): 
\quad 
\lambda_\bsPsi^2 \sum_{j\in\N} | c_j |^2 
\leq 
\left\|\sum_{j\in\N} c_j \psi_j \right\|_{\calX}^2 
\leq 
\Lambda_\bsPsi^2 \sum_{j\in\N} | c_j |^2 
\,.
\]
There holds  \emph{continuous-discrete equivalence}:
any $v\in \calX$ can be \emph{equivalently represented}
by the sequence 
$\bmc = (c_j)_{j\in \N}$ of its coefficients with respect to $\bsPsi$.

The canonical dual frame 
$\widetilde\bsPsi = (\widetilde\psi_j)_{j\in\N}$ of $\bsPsi$ 
is also a Riesz basis of $\calX$, which is referred to as the
\emph{dual basis} or the \emph{biorthogonal system} to $\bsPsi$, 
since for all $j\in\N$ and all $k\in\N$ holds $\dup{\psi_j}{\widetilde\psi_k}=\delta_{kj}$. 
We refer to \cite[Sct.~5]{Christensen} for a comprehensive treatment, and 
to \cite{DahmStevEBE,DavydovStevenson2006} for
constructions of piecewise polynomial Riesz bases for
Sobolev spaces in polytopal domains $\domain\subset {\mathbb R}^d$.
%%%%%%%%%%%%%%%%%%%%%%%%%%%%%%%%%%%%%%%%%%%%%%%%%%%%
\subsubsection{Orthonormal bases}
\label{sec:ONB}
%%%%%%%%%%%%%%%%%%%%%%%%%%%%%%%%%%%%%%%%%%%%%%%%%%%%
%
Orthonormal bases (ONBs) are particular instances of frames and Riesz
bases: 
if $\bsPsi$ is an orthonormal basis of $\calZ$, 
then $\widetilde\bsPsi=\bsPsi$. 
This covers ``principal-component'' shape-feature bases such as 
the shape encoding \eqref{eq:KL}
obtained
by principal component analyses associated with a covariance operator
corresponding to a Gaussian measure on $\calZ$ 
as commonly used in statistical learning theory (e.g.~\cite{SteinwartScovel12}).  
Such bases are generally not explicitly available, 
but may be approximately calculated in practice.
%
%%%%%%%%%%%%%%%%%%%%%%%%%%%%%%%%%%%%%%%%%%%%%%%%%%%%%%%%%%%
\subsection{Shape encoder and solution decoder}
\label{sec:EncDec}
%%%%%%%%%%%%%%%%%%%%%%%%%%%%%%%%%%%%%%%%%%%%%%%%%%%%%%%%%%%
In the following, we use the notation
  \begin{equation}\label{eq:psieta}
    \bsPsi_\calZ = (\psi_j)_{j\in\N}\,,\quad
    \widetilde\bsPsi_\calZ = (\widetilde\psi_j)_{j\in\N}\,,\quad
    \bsPsi_\calX = (\eta_j)_{j\in\N}\,,\quad
    \widetilde\bsPsi_\calX = (\widetilde\eta_j)_{j\in\N}
  \end{equation}
to denote frames and their canonical dual frames of $\calZ$ and of $\calX$, respectively.
With the corresponding analysis operators
$\rF_\calZ$, $\rF_\calX$,
the encoder/decoder pair in \eqref{eq:CDE} % -
is defined 
by the analysis and synthesis operators 
that are given by (cf.\ \eqref{eq:CDE})
\begin{equation}\label{eq:ED}
  \calE := {\widetilde\rF}_\calZ 
  = 
  \begin{cases}
    \calZ\to \ell^2(\N)\,,\\
    x\mapsto (\dup{x}{\widetilde\psi_j})_{j\in\N}\,,
  \end{cases}
  \quad
  \calD:=\rF_\calX'
  =
  \begin{cases}
    \ell^2(\N)\to \calX\,,\\
    (y_j)_{j\in\N}\mapsto \sum_{j\in\N}y_j\eta_j\,.
  \end{cases}  
\end{equation}
\begin{remark}
  If
  $\bsPsi_\calZ$, $\bsPsi_\calX$ are Riesz bases of $\calZ$, $\calX$,
  respectively, then the 
  shape encoder $\calE : \calZ \to\ell^2(\N)$ 
  and solution decoder $\calD :\ell^2(\N)\to\calX$ 
  in \eqref{eq:ED}, and their adjoints, are boundedly invertible. 
  The PCA type shape encoding in Section~\ref{sec:PrbFrm} is of this type.
\end{remark}
%%%%%%%%%%%%%%%%%%%%%%%%%%%%%%%%%%%%%%%%%%%%%%%%%%%%%%%%%%5
\subsection{Smoothness scales}
\label{sec:Scales}
%%%%%%%%%%%%%%%%%%%%%%%%%%%%%%%%%%%%%%%%%%%%%%%%%%%%%%%%%% 5
Our analysis will require 
subspaces of $\calZ$ and $\calX$, $\calY$
exhibiting ``extra smoothness'', corresponding
to smoother shape variations and solutions with higher regularity.
We formalize this with the positive, monotonically
decreasing weight sequence 
$\bsw=(w_j)_{j\in\N} \in (0,1]^\N$ from \eqref{eq:wght}
such that
$\bsw^{1+\varepsilon}\in\ell^{1}(\N)$ for all $\varepsilon>0$. 

With $\bsw$ and the representation systems \eqref{eq:psieta},
for $s, t\ge 0$
we introduce scales of Hilbert spaces
$\calZ^s\subset \calZ$, $\calX^t \subset \calX$ 
with norms
\begin{equation}\label{eq:XsYt}
  \norm[\calZ^s]{z}^2:=\sum_{j\in\N}\dup{z}{\widetilde{\psi}_{j}}^2 w_{j}^{-2s}\,,
  \qquad
  \norm[\calX^t]{x}^2:=\sum_{j\in\N}\dup{x}{\widetilde{\eta}_{j}}^2 w_{j}^{-2t}\,.
\end{equation}

  \begin{lemma} \label{lem:cXs}
    For $s\ge 0$,
    $\calZ^s=\set{z\in\calZ}{\norm[\calZ^s]{z}<\infty}$ is a Hilbert space
    with inner product
    $\dup{z}{z'}_{\calZ^s} = \sum_{j\in\N}\,\dup{z}{\widetilde\psi_j}\,\dup{z'}{\widetilde\psi_j}w_j^{-2s}$.
  \end{lemma}
  \begin{proof}
    Clearly $\dup{\cdot}{\cdot}_{\calZ^s}$ defines an inner product on
    the set $\calZ^s$ compatible with the norm $\norm[\calZ^s]{\cdot}$. 
    Then, $\calZ^s$ is closed with respect to this norm (see, e.g., 
    \cite[Lemma~1]{HSZ}).
\end{proof}
\begin{remark}\label{rmk:weighted}
  For ONBs 
  $\bsPsi_\calZ = (\psi_j)_{j\in\N}$ and
  $\bsPsi_\calX = (\eta_j)_{j\in \N}$ of $\calZ$ and $\calX$, 
  the sequences $(w_j^s\psi_j)_{j\in\N}$ and $(w_j^t\eta_j)_{j\in\N}$ 
  form ONBs of $\calZ^s$ and $\calX^t$, respectively.
\end{remark}
\begin{remark}\label{remk:SplWavD}
The biorthogonal spline wavelet bases constructed in \cite{DahmStevEBE} 
are continuous, piecewise polynomial,
and are stable in the Sobolev spaces $H^s(\domain)$ for $|s|< 3/2$. 
Scaling $\psi_j$ obtained in \cite{DahmStevEBE} so that the Riesz basis
property holds in $\calX = L^2(\domain)$, 
appropriate choices of the weight sequence $w_j$ provides 
in particular \eqref{eq:XsYt} with $0\leq s < 3/2$ in $\calX^s = H^s(\domain)$.
In addition, the wavelet constructions $\psi_j$ and $\widetilde{\psi}_j$ 
in \cite{DahmStevEBE}
can furnish wavelet systems with vanishing moments of any a-priori required 
polynomial order.
\end{remark}

With \eqref{eq:XsYt}, 
we quantify the smoothness of the parametric vector field 
${\bf y}\mapsto {\bf V}_{{\bf y}}$ in \eqref{eq:KL}
via weighted summability with respect to the weight sequence $\bsw$.
%
%==========================================================
\subsection{Error bounds for neural and spectral operator surrogates}
\label{sec:ErrBdNrSpcOp}
%==========================================================
Neural and spectral operators 
approximate maps $\calG$ from (subsets of) $\calZ$ into $\calX$. 
In the following denote
with $\bsPsi_\calZ$, $\bsPsi_\calX$ 
fixed frames of the separable Hilbert spaces $\calZ$, $\calX$.
 
We recall Assumption~\ref{ass:V} which 
stipulates an affine-parametric dependence of the vector field 
${\bf V}_{\bf y}: D_\refd \to D_{\bf y}$ generating admissible shapes $\bcalS$. 
As PCAs $({\boldsymbol\varphi}_k)_{k\in\N}$ as in \eqref{eq:KL} 
are ONBs as in Section~\ref{sec:ONB},
the shape-encoder $\calE$ is realized by innerproducts of 
${\bf V}$ with the principal components ${\boldsymbol\varphi}_j$ in \eqref{eq:KL}
in the inner product of $\calZ$. 
Accordingly, 
in \eqref{eq:ED} due to Assumption~\ref{ass:V} we choose
$\psi_j = \widetilde{\psi}_j = {\boldsymbol\varphi}_j$
and the shape-encoder $\calE:\calZ\to\ell^2(\N)$, 
and the solution-decoder $\calD:\ell^2(\N)\to\calX$ 
as in \eqref{eq:ED}.
For expression rate estimates, 
we require regularity of admissible shapes $\bcalS$.
We quantify this by imposing weighted summability 
on the parameter sequences $(y_j)_{j\in\N}$
furnished by the shape-encoder \eqref{eq:KL}.
To this end, for 
    $$
    U:=[-1,1]^\N\,,
    $$
    and 
    $\bsw$ as in Section \ref{sec:Scales}, for $s > \frac12$,
    and a scaling factor $r>0$, 
    we introduce the \emph{co-ordinate scaling map}
\begin{equation}\label{eq:sigma}
  \sigma_r^s : U\to \calZ : \bsy\mapsto r \sum_{j\in\N}w_j^s y_j {\boldsymbol\varphi}_j\,.
  \end{equation}
  The condition $s> \frac12$ ensures that the coefficient
  sequence $(rw_j^sy_j)_{j\in\N}$ belongs to $\ell^2(\N)$ so that
  $\sigma_r^s$ is well-defined as a mapping from $U$ to $\calZ$ (cp.~\eqref{eq:calFp}). 
  With the scaling map $\sigma^s_r$ in \eqref{eq:sigma},
  we introduce ``Cube'' subsets of $\calZ$ via
\[
  C_r^s(\calZ) := \{ \sigma_r^s(\bsy):\, {\bf y} \in U\}\,.
\]
The sets $C_r^s(\calZ)$ depend, via \eqref{eq:sigma}, 
on the representation system $\bsPsi = (\psi_j)_{j\in\N}$ 
and on the weight sequence $\bsw$.
For $s\geq 0$, 
$C_r^s(\calZ)$ gives rise to collections $\bcalS$ of admissible shapes
\begin{equation}\label{eq:Crs}
  \bcalS = \bcalS(D_\refd, \bsPsi, \bsw, r,s)  
:= 
\left\{ {\bf V}(D_\refd) \mid {\bf V}  \in C_r^s(\calZ) \right\} \,.
%: {\calE({\bf V})\in\bigtimes_{j\in\N}[-rw_j^{s},rw_j^{s}]\bigg\}
%  \\
%               &\;=
%  \setlr{{\bf V} \in \calZ}{\sup_{j\in\N}|\dup{{\bf V}}{\widetilde{\psi}_j}|w_j^{-s}\le r}.
\end{equation}
\begin{remark} \label{rmk:Cs_in_Xs_prime}
Let $s'\ge 0$ and $s>s'+1/2$.
Then
$C^s_r(\calZ) \subset \calZ^{s'}$, 
since for $a\in C^s_r(\calZ)$
\begin{equation*}
\| {\bf V} \|^2_{\calZ^{s'}}
=
\sum_{j\in\N}
\langle {\bf V}, {\boldsymbol\varphi}_j \rangle ^2 w_j^{-2s'}
\leq 
r^2 \sum_{j\in\N} w_j^{2(s-s')} <\infty\,, 
\end{equation*}
due to $(w_j)_{j\in\N}\in\ell^{1+\eps}(\N)$ for any $\eps>0$\,.
\end{remark}
We shall work under the \emph{assumption that the shape-to-solution map
$\calG$ allows a complex differentiable extension to some open superset of 
$\widetilde C_r^{s}(\calZ)$ in $\calZ_\C$}:

\begin{assumption}\label{ass:1}
  There exist $r>0$, $s>1$, $t>0$, $M<\infty$ and 
  an open set $O_\C\subseteq \calZ_\C$
  containing $C_r^{s}(\calZ)$ such that
  $\sup_{{\bf V} \in O_\C}\norm[\calX_\C^t]{\calG({\bf V} )}\le M$
  and $\calG:O_\C \to \calX_\C$ is holomorphic.
\end{assumption}
  Assumption \ref{ass:1} is for instance satisfied by 
  shape-holomorphic solution
  operators corresponding to second order elliptic PDEs, 
  especially for the specific examples given in Section~\ref{sec:Expls}.
  It is merely required with respect to the topology of $\calZ_\C$ 
  and not with respect to the stronger topology of $\calZ_\C^t$. 
  The assumed boundedness already implies holomorphy of 
  $\calG:O_\C\to\calX_\C^{t'}$ for any $t'\in [0,t)$.
%%%%%%%%%%%%%%%%%%%%%%%%%%%%%%%%%%%%%%%%%%%%%%%%%%%%%%%%%%%%%%%%%%%%%%%%%%%%%%%%%%%%%%%%%
\subsubsection{Worst-case error for NN operator surrogates}
\label{sec:WCENNOp}
%%%%%%%%%%%%%%%%%%%%%%%%%%%%%%%%%%%%%%%%%%%%%%%%%%%%%%%%%%%%%%%%%%%%%%%%%%%%%%%%%%%%%%%%%
Our first main result states that a
holomorphic operator $\calG$ as in Assumption \ref{ass:1} can be
uniformly approximated on $C_r^s(\calX)$ by a NN surrogate of the form
$\calD\circ\widetilde\bG\circ\calE$, where $\widetilde\bG$ is a ReLU NN.  
More precisely, $\widetilde\bG$ is a map of the form
\begin{equation}\label{eq:NN}
  {\widetilde\bG =} A_L\circ{\rm ReLU}\circ A_{L-1}\circ\cdots
  \circ A_{1}\circ {\rm ReLU}\circ A_0\,,
\end{equation}
where the application of ${\rm ReLU}(x):=\max\{0,x\}$ is understood
componentwise, and each $A_j:\R^{n_j}\to\R^{n_{j+1}}$ is an affine
transformation of the form $A_j(x)=W_jx+b_j$ with
$W_j\in\R^{n_{j+1}\times n_j}$, $b_j\in\R^{n_{j+1}}$.  
The weights $W_j$ and the biases $b_j$ determine the NN. 
The \emph{size $N$ of the NN in \eqref{eq:NN}} is
the number of nonzero entries of all $W_j$ and $b_j$, i.e.
\[
  N = {\rm size}(\widetilde\bG) := \sum_{0\leq j \leq L } \| W_j \|_0 + \| b_j \|_0 \,.
\]
\begin{remark}[Padding and Restriction]\label{rmk:padding}
  A NN $\widetilde\bG$ of the form \eqref{eq:NN}
  represents a map from $\R^{n_0}\to\R^{n_{L+1}}$.  
  We will also understand $\widetilde\bG$ as a map from $\ell^2(\N)\to\ell^2(\N)$. 
  To formalize, we introduce for $n\in \N$ the \emph{restriction map} 
$$
  \calR_n: \ell^2(\N) \to \R^n: (x_j)_{j\in\N} \mapsto (x_j)_{j=1,\ldots,n}
$$
and its adjoint, the extension by zero ``padding'' operation
$$
\calR_n': \R^n \to  \ell^2(\N): (x_j)_{j=1,\ldots,n} \mapsto (x_1,\ldots,x_n,0,0,\ldots)\,.
$$
With these operations, 
the approximator NN $\widetilde\bG$ in \eqref{eq:NN} 
becomes a mapping from $\ell^2(\N)\to\ell^2(\N)$, i.e.
$$ 
\widetilde{\bG}_N := \calR'_{n_L+1} \circ \widetilde{\bG} \circ \calR_{n_0}: \ell^2(\N) \to \ell^2(\N) \,.
$$
A \emph{Shape Transfer Operator Network} 
is obtained by combining $\widetilde{\bG}$ with shape encoder $\calE$ based on \eqref{eq:KL}
and PDE solution decoders, i.e.
\begin{equation}\label{eq:ONet}
\widetilde{\calG}_N 
:= 
\calD \circ \widetilde{\bG}_N \circ \calE 
= 
\calD \circ \calR'_{n_L+1} \circ \widetilde{\bG} \circ \calR_{n_0} \circ \calE : \calX\to \calY\,.
\end{equation}
\end{remark}
In the present setting, 
$\calE$ can be based for example 
on the PCA \eqref{eq:KL}, 
while the decoder 
$\calD\in \calL\big(\ell^2(\N),\calX\big)$ 
can be a frame-synthesis operator \eqref{eq:calFp}, 
as realized e.g. by 
B-splines or multi-level Galerkin Finite Element approximation in $D_\refd$.
The following result provides expression rate for such architectures.
\begin{theorem}
  \label{thm:main}
  Let Assumption \ref{ass:1} be satisfied with $r>0$, $s>1$, $t>0$.  
  Fix $\delta>0$ (arbitrarily small).
  Then there exists a constant $C > 0$ 
  such that
  for every $N\in\N$ there exists a ReLU NN
  $\widetilde{\bG}_N$ of size $O(N)$ such that
\[
    \sup_{{\bf V} \in C_r^s(\calZ)}\big\|\calG({\bf V}) - \calD\big({\widetilde{\bG}}_N(\calE({\bf V} ))\big)\big\|_{\calX}
    \le 
    C N^{-\min\{s-1,t\}+\delta}\,.
\]
\end{theorem}

Introduce the closed ball of radius $r$ in $\calZ^s$
\[
  B_r(\calZ^s) := \setlr{{\bf V} \in\calZ}{\norm[\calZ^s]{{\bf V}}\le r}.
\]
For any $\eps>0$, 
$ B_r(\calZ^s)\subseteq C_r^s(\calZ) 
  \subseteq B_{r_\eps}(\calZ^{s-\frac{1}{2}-\eps})$
with 
$r_\eps:=r\big(\sum_{j\in\N}w_j^{1+2\eps}\big)^{1/2}<\infty$
(cp.~\eqref{eq:XsYt}, \eqref{eq:Crs} and Remark~\ref{rmk:Cs_in_Xs_prime}). 
\begin{corollary}\label{cor:ball}
  Consider the setting of Theorem \ref{thm:main}. 
  Then there exists a constant $C>0$ such that for every $N\in \N$
  there exists a ReLU NN $\widetilde{\bG}_N$ of size $O(N)$ such that
\[
    \sup_{{\bf V} \in B_r(\calZ^s)} \big\|\calG({\bf V})-\calD\big((\widetilde{\bG}_N(\calE({\bf V} ))\big)\big\|_{\calX} 
    \le 
    C N^{-\min\{s-1,t\}+\delta}\,.
\]
\end{corollary}
%%%%%%%%%%%%%%%%%%%%%%%%%%%%%%%%%%%%%%%%%%%%%%%%%%%%%%%%
\subsubsection{Mean-square error for NN operator surrogates}
\label{sec:mnsq}
%%%%%%%%%%%%%%%%%%%%%%%%%%%%%%%%%%%%%%%%%%%%%%%%%%%%%%%%
We can improve the operator approximation rate of Theorem
\ref{thm:main}, if we measure the error in a mean-square sense. 
This error measure is natural
in context of regression learning of shape parametrizations from data.
Assume that $\bsPsi_\calX$ is a Riesz basis, and let
$\mu:=\otimes_{j\in\N}\frac{\lambda}{2}$ be the uniform probability
measure on 
$U:=\bigtimes_{j\in\N}[-1,1]$ equipped with its 
product Borel sigma algebra, where 
$\lambda$ stands for the Lebesgue measure in $\R$. 
The pushforward $(\sigma_r^s)_\sharp \mu$ of $\mu$ under $\sigma_r^s$
then constitutes a measure on $C_r^s(\calX)$, cf.~\cite[Rmk.~10]{HSZ}.
\begin{theorem} \label{thm:main_riesz}
  Assume that $\bsPsi_\calX$ is a Riesz basis.  
  Let Assumption \ref{ass:1} be satisfied with $r>0$, $s>1$, $t>0$. 
  Fix $\delta>0$ (arbitrarily small).
  Then there exists a constant $C > 0$ 
  such that
  for every $N\in\N$ there exists a ReLU NN
  $\widetilde{\bG}_N$ of size $O(N)$ such that
\[
    \normlr[L^2(C_r^s(\calZ),(\sigma_r^s)_\sharp\mu;\calX)]{\calG - \calD\circ {\widetilde{\bG}}_N \circ \calE}
    \le 
    C N^{-\min\{s-\frac{1}{2},t\}+\delta}\,.
\]
\end{theorem}
%%%%%%%%%%%%%%%%%%%%%%%%%%%%%%%%%%%%%%%%%%%%%%%%%%%%%%%%%%%%%%%%%%%
\subsubsection{Worst-case error for spectral operator surrogates}
\label{sec:wcgpc}
Here, we replace the NN approximator $\widetilde\bG_N$ 
by a multivariate polynomial $p_N:\R^{n}\to\R^{m}$. 
The operator surrogate takes the form $\calD\circ p_N\circ \calE$, 
where the composition is
again understood as truncating the output of $\calE$ after the first $n$
parameters, and padding the output of $p_N$ with infinitely many zeros
(cp.~Rmk.~\ref{rmk:padding}). 
While spectral surrogates achieve the
same converence rate as NN-based approximators, 
for $p_N$ the proof is constructive (\cite{westermann2026performanceneuralpolynomialoperator,DNPS26_1162}):
one can explicitly compute $p_N$ as an interpolation polynomial,
based on a finite set of \emph{judiciously chosen} 
(rather than random, i.i.d.) input-output pairs.
Hence, the ``training'' 
consists of an explicit and deterministic construction,
rather than the minimization of a (typically non-convex)
loss by stochastic optimization methods.

\begin{theorem}[\cite{HSZ}]
  \label{thm:main2}
  Consider the setting of Theorem \ref{thm:main}. 
  Then
  there is a constant $C>0$ such that 
  for every $N\in\N$ there exists a multivariate polynomial $p_N$
  such that 
  \begin{equation*}
    \sup_{{\bf V} \in C_r^s(\calZ)}\big\|\calG({\bf V})-\calD\big(p_N(\calE({\bf V}))\big)\big\|_{\calX}
    \le 
     C N^{-\min\{s-1,t\}+\delta}\,.
  \end{equation*}
  Furthermore, $p_N$ belongs to an $N$-dimensional space of
  multivariate polynomials. Its components are interpolation
  polynomials, whose computation requires the evaluation of
  $\dup{\calG({\bf V})}{\widetilde{\eta}_j}$ 
  in at most $N$ tuples
  $({\bf V} , j)\in C_r^s(\calZ)\times\N$.
\end{theorem}
\begin{remark}[Shape Transfer Learning]
\label{rmk:ShTrLrn}
The present results provide error bounds for so-called
``shape-transfer learning''. 
This describes the following setting:
Assume given domains 
$D_1 = \Phi_1(D_\refd)$ and $D_2 = \Phi_2(D_\refd)$
which are (assumed to be) diffeomorphic to one common 
reference domain $D_\refd$, which is ``latent'', i.e.,
it may never be physically realized.
Here, $\Phi_i$ denotes the solution-transport under parametric
vector fields ${\bf V}_i$ 
from one common reference domain $D_\refd$.

Assuming the parametrization in Section~\ref{sec:PrbFrm}, 
Assumption~\ref{ass:V}, we shall use $D_\refd = D_\bsnul$, and
apply the preceding theory twice. 
We thereby conclude that the solution transfer from $D_1$ to $D_2$ 
under 
$\Phi_{12} := \Phi_2 \circ \Phi_1^{-1}$
can be learned at the same rates
with the same ONet architecture.
\end{remark}
%==========================================================
\section{Conclusion}
\label{sec:Conclusion}
%==========================================================
Conclusions and generalizations of this work: 
We used analyticity in order to invoke the approximation
rate bounds for Operator Network emulations of holomorphic
maps between function spaces from \cite{HSZ}. 
Naturally, the argument in the proof of Theorem~\ref{thm:regu}
allows, with the composition lemma for Gevrey-$s$ regular maps
from \cite{HSS}, also Gevrey-$s$ smooth regularity of, 
e.g., the parameter-to-shape map ${\bf V}$ in Assumption \ref{ass:V}.
\begin{remark} \label{rmk:CplxCoeff} 
The results of this article can be
extended to the case where one or both of $\calZ$ and $\calX$ are separable Hilbert
spaces over the coefficient field $\C$. 
Complex $\calX$ is, for example, 
natural in time-harmonic acoustic or magnetic field modelling \cite{H3S,HSS24_1103}.
The Hilbert spaces $\calZ^s$ and $\calX^t$ are included in their 
(unique, see \cite{MunozCplxB}) complexified versions
$\calZ_\C^s = \{1,i\}\otimes \calZ^s$ and
$\calX_\C^t = \{1,i\}\otimes \calX^t$.
Encoder and decoder in \eqref{eq:ED} then 
act on weighted, complex sequence spaces.
\end{remark}
The presently developed, abstract framework provides expression rate bounds for 
shape neural operators. It covers many recent algorithmic approaches, as e.g.
the so-called ``Shape-DINO'' \cite{gong2026shapederivativeinformedneuraloperators} approach,
shape-FNOs 
\cite{LiHuangFNOGeo23,
li2023geometryinformedneuraloperatorlargescale,
loeffler2024graphfourierneuralkernels},
digital twins \cite{liu2024deepneuraloperatorenabled}
which will give rise in present setting to shape-derivative informed operator learning.
We admitted abstract frame-based shape encoders, which cover many practical constructions:
PCA, B-splines, IGA. 
Certain non-affine, non-linear encoders such as variational autoencoders
as used in the DIVA approach \cite{DIVA} require an extension of the present concept.

Concluding the list of contributions, 
the presently demonstrated
results are built on differential-geometric concepts of domains being
diffeomorphic to a (not necessarily physically realizable) \emph{reference or Lagrangian configuration},
a standard concept in Truesdellian continuum mechanics \cite{TruesdellV1,Valent}.
The present results are \emph{not covering} 
shape-representations given by unstructured, ``point cloud'' geometry data.
As discussed e.g.~in \cite{zeng2025pointcloudneuraloperator},
a mathematical analysis for these approaches will rely on measure-transport concepts,
see e.g.~\cite{li2025geometricoperatorlearningoptimal,LiHuangFNOGeo23}. 
Details in the presently considered, shape-holomorphic case, will be developed elsewhere.
Numerical experiments with detailed studies of expressivity of neural and spectral 
operator surrogates are available in \cite{westermann2026performanceneuralpolynomialoperator}.

\subsection*{Acknowledgement}
This research was funded by the Swiss National Science Foundation (SNSF) 
and the Vietnam National Foundation for Science and Technology Development 
(NAFOSTED) through the Vietnamese-Swiss Joint Research Project IZVSZ2\_229568.
%==========================================================
\bibliographystyle{plain}
\bibliography{bibl}
\end{document}